\documentclass[twoside,11pt]{article}

\usepackage[preprint]{jmlr2eArxiv}
\usepackage{amsmath,amssymb,amsfonts}
\usepackage{graphicx}
\usepackage{textcomp}
\usepackage{xcolor}
\usepackage{hyperref}
\usepackage{amsbsy}
\usepackage{algorithm,algpseudocode}
\usepackage{subcaption}
\usepackage{enumerate}
\usepackage{algpseudocode}
\usepackage{float}

\usepackage{ragged2e} 
\usepackage{tabularx} 
\algnewcommand{\Inputs}[1]{%
	\State \textbf{Inputs:}
	\Statex \hspace*{\algorithmicindent}\parbox[t]{.8\linewidth}{\raggedright #1}
}
\algnewcommand{\Initialize}[1]{%
	\State \textbf{Initialize:}
	\Statex \hspace*{\algorithmicindent}\parbox[t]{.8\linewidth}{\raggedright #1}
}

\providecommand{\abs}[1]{\lvert#1\rvert}
\renewcommand{\b}[1]{\ensuremath{\mathbf{#1}}} 
\newcommand{\E}[1]{\ensuremath{\mathbb{E}\left[#1\right]}}  
\newcommand{\Ex}[1]{\ensuremath{\EE[#1]}}  
\newcommand{\Et}[1]{\ensuremath{\mathbb{E}_t[#1]}}  
\newcommand{\ind}{1\hspace{-1.6mm}1} 
\newcommand{\norm}[1]{\ensuremath{\left\|#1\right\|}} 
\providecommand{\tr}[1]{\mathrm{tr}\left(#1\right)} 
\providecommand{\ip}[2]{\langle #1, #2 \rangle} 
\newcommand{\col}[1]{\textcolor{blue}{#1}} 

\renewcommand{\O}[1]{\mathcal{O}}

\def \x {{\b{x}}}

\def \e {{\b{e}}}

\def \grd {{\text{grad}}}

\def \EE {{\mathbb{E}}}

\def \G {{\mathbf{G}}}
\def \S {{\mathbf{S}}}

\def \cB {{\mathcal{B}}}
\def \cF {{\mathcal{F}}}
\def \O {{\mathcal{O}}}

\def \cL {{\mathcal{L}}}
\def \cE {{\mathcal{E}}}

\def \cI {{\mathcal{I}}}
\def \cN {{\mathcal{N}}}

\def \A {{\mathbf{A}}}
\def \B {{\mathbf{B}}}
\def \C {{\mathbf{C}}}
\def \D {{\mathbf{D}}}
\def \E {{\mathbf{E}}}
\def \F {{\mathbf{F}}}
\def \G {{\mathbf{G}}}
\def \H {{\mathbf{H}}}
\def \I {{\mathbf{I}}}
\def \M {{\mathbf{M}}}
\def \P {{\mathbf{P}}}
\def \Q {{\mathbf{Q}}}
\def \R {{\mathbf{R}}}
\def \S {{\mathbf{S}}}
\def \U {{\mathbf{U}}}
\def \V {{\mathbf{V}}}
\def \W {{\mathbf{W}}}
\def \X {{\mathbf{X}}}
\def \Y {{\mathbf{Y}}}

\newcommand{\Bup}[1]{\B^{\mathrm{up}}_{#1}}

\providecommand{\ipx}[2]{\langle #1, #2 \rangle_{\X}} 

\def \xib {{\boldsymbol{\xi}}}

\def \etab {{\boldsymbol{\eta}}}

\def \T {{\mathsf{T}}}

\def \mub {{\boldsymbol{\mu}}}

\def \Rn {{\mathbb{R}}}
\def \Pn {{\mathbb{P}^n}}

\def \cM {{\mathcal{M}}}

\def \Sig {{\boldsymbol{\Sigma}}}

\def \Expx {{\mathrm{Exp}_\X}}
\def \Expxt {{\mathrm{Exp}_{\X_t}}}
\def \Eb {{\mathbf{E}}}

\def \Rar {{\Rightarrow}}

\newtheorem{assumption}{}

\allowdisplaybreaks

\usepackage{lastpage}

\ShortHeadings{Low-complexity subspace-descent over symmetric positive definite manifold}{Yogesh Darmwal and Ketan Rajawat}
\firstpageno{1}

\begin{document}

\title{Low-complexity subspace-descent over symmetric positive definite manifold}

\author{\name Yogesh Darmwal \email yogeshd@iitk.ac.in \\
       \addr Dept. of Electrical Engineering\\       
       Indian Institute of Technology Kanpur\\
       Uttar Pradesh -208016, India
       \AND
       \name Ketan Rajawat \email ketan@iitk.ac.in \\
       \addr Dept. of Electrical Engineering\\       
       Indian Institute of Technology Kanpur\\
       Uttar Pradesh -208016, India}
\maketitle

\begin{abstract}
This work puts forth low-complexity Riemannian subspace descent algorithms for the minimization of functions over the symmetric positive definite (SPD) manifold. Different from the existing Riemannian gradient descent variants, the proposed approach utilizes carefully chosen subspaces that allow the update to be written as a product of the Cholesky factor of the iterate and a sparse matrix. The resulting updates avoid the costly matrix operations like matrix exponentiation and dense matrix multiplication, which are generally required in almost all other Riemannian optimization algorithms on SPD manifold. We further identify a broad class of functions, arising in diverse applications, such as kernel matrix learning, covariance estimation of Gaussian distributions, maximum likelihood parameter estimation of elliptically contoured distributions, and parameter estimation in Gaussian mixture model problems, over which the Riemannian gradients can be calculated efficiently. The proposed uni-directional and multi-directional Riemannian subspace descent variants incur per-iteration complexities of $\O(n)$ and $\O(n^2)$ respectively, as compared to the $\O(n^3)$ or higher complexity incurred by all existing Riemannian gradient descent variants. The superior runtime and low per-iteration complexity of the proposed algorithms is also demonstrated via numerical tests on large-scale covariance estimation and matrix square root problems.  MATLAB code implementation is publicly available on GitHub : \href{https://github.com/yogeshd-iitk/subspace_descent_over_SPD_manifold}{\col{https://github.com/yogeshd-iitk/subspace\_descent\_over\_SPD\_manifold}}
\end{abstract}


\section{Introduction}
We consider the following optimization problem
\begin{align}
	\min_{\X \in \Pn} f(\X)\label{problem}
\end{align}
where $\Pn$ is the set of $n \times n$ real symmetric positive definite (SPD) matrices and $f:\Pn \rightarrow \Rn$ is a geodesically convex and smooth function on $\Pn$ (\cite{zhang16PMLR}); see Definition \ref{paper.def1}. 
Such problems arise in kernel matrix learning (\cite{li2009kernel}), covariance estimation of Gaussian distributions (\cite{Wiesel2012geodesic,wiesel2015structured,zhang2013multivariate}), maximum-likelihood parameter estimation of Elliptically Contoured Distributions (ECD) (\cite{sra2013geometric}), parameter estimation in Gaussian mixture model (\cite{Hosseini2015}),  matrix square root (\cite{higham1997stable}, \cite{jain2017global}, \cite{sra2015matrix}, \cite{gawlik2019zolotarev}, \cite{oviedo2020scaled}) and matrix geometric mean estimation (\cite{Moakher2005Differential,bhatia2006means, bhatia2013riemannian,weber2020riemannian,ZhangNIPS2016,LiuNIPS2017,zhang16PMLR}). We seek to leverage the structure of $f$ and $\Pn$ in order to design computationally efficient manifold optimization algorithms for solving \eqref{problem}.

Standard approaches to solve \eqref{problem}, relying on projections onto $\Pn$ or on the interior point method (\cite{benson2003solving}), become increasingly difficult to implement as $n$ increases. Recent years have witnessed the development of Riemannian optimization methods, which seek to be more efficient and scalable by utilizing the geoemtry of $\Pn$ (\cite{zhang16PMLR,ZhangNIPS2016,LiuNIPS2017}). However, most state-of-the-art Riemannian optimization methods, including the Riemannian gradient descent and its variants, involve matrix exponentiation, matrix inversion, matrix square root, and dense matrix multiplication at each iteration.
Consequently, these algorithms incur a per-iteration complexity of $\O(n^3)$, which might be prohibitive in large-scale settings, such as those encountered in various applications listed earlier. 
The computational and memory limitations of modern computers motivate the need for large-scale optimization algorithms with per-iteration complexity that is quadratic or even linear in $n$. 

Coordinate (subspace) descent methods have been widely applied to solve large-scale problems in the Euclidean space (\cite{bertsekas2016nonlinear, luenberger2015linear}). Coordinate-wise operations are particularly attractive for huge-scale problems, where full-dimensional vector operations have prohibitive computational and memory requirements (\cite{Nesterov2012cd, saha2013nonasymptotic, beck2015cyclic, karimi2016linear, richtarik2014iteration, Fercoq2015acc}). A careful observation reveals that the choice of basis vectors (canonical basis in Euclidean space) and the linear nature of update equations lead to such simplicity of subspace descent algorithms in the Euclidean setting.

Coordinate descent algorithms, when directly adapted to Riemannian manifolds, lose their simplicity. Since manifolds do not possess vector space structure, the non-linear nature of the updates prevents us from partitioning the coordinates into blocks that can be individually updated. Within this context, the work in \cite{gutman2022coordinate} develops a Reimannian coordinate descent algorithm for general manifolds, but the proposed approach does not always lead to computationally simpler algorithms for solving \eqref{problem}.
We propose the Riemannian subspace descent algorithm that achieves a lower per-iteration computational cost by maintaining and directly updating Cholesky factors of the iterates. The different variants of the proposed algorithm allow selecting one or more directions, either randomly or greedily. 

The per-iteration complexity of any Riemannian first order optimization algorithm can be seen as consisting of two components: cost of calculating the Riemannian gradient and the cost of carrying out the update, which may involve complicated operations such as matrix exponentiation. Hence, to achieve a lower per-iteration complexity, we focus on the following class of functions: 
\begin{align}
	\cF=\left\{f\Bigg|f\left(\boldsymbol{\X}\right)=g\left(\begin{array}{ll} \left\{\tr{\C_p\X^{-1}}\right\}_{1\leq p\leq  P},\left\{\tr{\D_q \X}\right\}_{1\leq q\leq Q},\log\det \X,\\
		\left\{\tr{\X\A_r\X\H_r}\right\}_{1\leq r\leq R}, \left\{\tr{\X^{-1}\F_s\X^{-1}\G_s}\right\}_{1\leq s\leq S},
		\\
		\left\{\tr{\P_m\X\Q_m\X^{-1}}\right\}_{1\leq m\leq M}
	\end{array}
	\right)\right\} \label{sc_fun_class}.
\end{align}
 
Interestingly, $\cF$ is rich enough to include a variety of geodesically convex and non-convex functions, and covers all the applications listed earlier except matrix geometric mean problem of $N$ SPD matrices for $N>2$. For $N=2$, the matrix geometric mean of two matrices $\W_1$ and $\W_2$ is equal to the minimizer of the function (see Appendix \ref{MGM_two}) 
\begin{align}
	f(\X)=\sum_{i=1}^{N}\tr{\W_i\X^{-1}+\W_i^{-1}\X}
\end{align}
which lies in the function class $\cF$.

To simplify the update equation, we begin by identifying a ``canonical" set of orthogonal basis vectors for the tangent space at a given point. Subsequently, the update direction is carefully selected so as to avoid dense matrix exponentiation operation in the update equation. The resulting update takes the form of multiplying the Cholesky factor of the iterate with a sparse matrix, which is amenable to efficient implementation. 

The base version of the proposed algorithm is the uni-directional randomized Riemannian subspace descent (RRSD) algorithm that randomly selects a single (out of possible $\O(n^2)$) subspace direction at every iteration, and incurs only $\O(n)$ complexity per-iteration. Its multi-directional variant selects $\O(n)$ dimensional subspace at every iteration, and incurs $O(n^2)$ complexity per-iteration. We also propose greedy subspace selection rules, resulting in the Riemannian greedy subspace descent (RGSD) algorithm, whose uni-directional and multi-directional variants incur per-iteration complexities of $\O(n^2)$ and $\O(n^2\log(n))$, respectively.



The key results are summarized in the Table \ref{table_result_comapre}. In the table, $t$ represents total number of iterates, $\mu$ represents strong-convexity parameter, $L$ represents Lipschitz-smoothness parameter, and $c$ is the constant dependent on the diameter and sectional curvature lower bound (\cite{zhang16PMLR}). In our case, the factor $\left(\frac{1}{n}\right)$ is due to the fact that the proposed algorithms are subspace descent algorithms, which work in an $\O(n)$-dimensional subspace, whereas Riemannian gradient descent (RGD) and Riemannian accelerated gradient descent (RAGD) work in the full $\big(\frac{n(n+1)}{2}\big)$-dimensional tangent space of the manifold $\Pn$.
\begin{table}[h!]
	\begin{center}
		\begin{tabular}{|p{5cm}|p{4cm}|p{2.5cm}|p{2cm}|} 
			\hline
			Algorithm& Convergence rate & Per-iteration complexity& Subspace-dimension \\
			\hline
			RRSD (proposed) &  $\mathcal{O}\left(\left(1-\frac{\mu}{4nL}\right)^t \right)$   \hfill&  \hfil $\O(n^2)$      &     \hfil $\O(n)$  \\ 
			\hline
			RGSD (proposed) &  $\mathcal{O}\left(\left(1-\frac{\mu}{8nL}\right)^t \right)$   \hfill&  \hfil $\O(n^2\log(n))$      &     \hfil $\O(n)$  \\ 
			\hline			
			RGD \cite{ZhangNIPS2016}  & $\mathcal{O}\left(\left(1-\min\left\{\frac{1}{c},\frac{\mu}{L}\right\}\right)^t\right)$    &    \hfil $\O(n^3)$    & \hfil    $\O(n^2)$  \\
			\hline
			RAGD \cite{zhang18estimate} & $\mathcal{O}\left(\left(1-\frac{9}{10}\sqrt{\frac{\mu}{L}}\right)^t\right)$    &   \hfil   $\O(n^3)$    &     \hfil $\O(n^2)$  \\
			\hline	
		\end{tabular}
	\end{center}
	\caption{Comparison of  proposed multi-directional Riemannian subspace-descent algorithms with RGD and RAGD algorithms for the class of strongly convex functions in \eqref{sc_fun_class}.	}\label{table_result_comapre}
\end{table}

\subsection{Related work}
A thorough analysis of geodesically convex functions over the manifold $\Pn$ is available in \cite{Sra2015conic}. Complexity of several first order Riemannian optimization algorithms over Hadamard manifolds was first provided in  \cite{zhang16PMLR}. Of particular importance is the Riemannian stochastic gradient descent (RSGD) algorithm proposed in \cite{zhang16PMLR}, which for problems of the form
\begin{align}
	\min_{\X \in \Pn} \sum_{s=1}^{S}f_p(\X)
\end{align}
uses the updates
\begin{align}
	\X_{t+1}
	&= \mathbf{\X_t}^{1/2}\exp\left(-\alpha_tS\X_t^{-1/2} \grd^R f_s\left(\X_t\right)\X_t^{-1/2}\right)\X_t^{1/2} \label{Riemann_SGD}
\end{align}
where $s$ is uniformly selected index from the index set $\left\{1,2,\dots,S\right\}$ and $\grd^R f_s$ is Riemannian gradient of function $f_s$; see Definition \ref{Riemann_grad}. Although the RSGD algorithm incurs lower complexity than the RGD algorithm, it still uses matrix exponentiation and matrix multiplication at every iteration, and therefore incurs a per-iteration complexity of $\mathcal{O}(n^3)$.

The first accelerated first-order algorithm for geodesically convex functions was proposed in \cite{LiuNIPS2017} by extending Nesterov's acceleration idea to nonlinear spaces. One limitation of the proposed algorithm, as pointed out in \cite{zhang18estimate}, is its dependence on an exact solution to a nonlinear equation at every iteration. To address this issue, \cite{zhang18estimate} proposed a tractable accelerated algorithm for the geodesically strongly convex function class. An extension to broader classes of geodesically convex and weakly-quasi-convex functions is also proposed in \cite{alimisis21momentum}. Finally, stochastic variants of accelerated RGD have likewise also been developed \cite{Bonnabel2013stoch, zhang2018r, ZhangNIPS2016, kasai2016riemannian, tripuraneni2018averaging, MASAGABabanezhad2019, hosseini2020recent, ahn2020nestero}. Due to the non-linearity of manifolds and the fact that the tangent space $T_{\X}\mathcal{M}$ at a point ${\X}$ is different from the tangent space $T_{\mathbf{Y}}\mathcal{M}$ at ${\mathbf{Y}} \neq \X$, all accelerated RGD variants require computationally intensive matrix exponentiation and parallel transport steps \cite{Sra2015conic, zhang18estimate} at every iteration. 

The first attempts to extend the idea of coordinate descent to the Stiefel manifold were made in \cite{celledoni2008descent, shalit2014coordinate, gao2018new}. The first work in the direction of extending the coordinate descent algorithm to general manifolds is \cite{gutman2022coordinate}, which takes its motivation from the Block Coordinate Descent (BCD) algorithm on Euclidean space. As in BCD, the tangent subspace descent (TSD) algorithm of \cite{gutman2022coordinate} works by first choosing a suitable tangent subspace, projecting the gradient onto it, and then taking a descent update step in the resulting direction. The approach in  \cite{gutman2022coordinate} can be computationally expensive because it requires multiple parallel transport operations in the deterministic case and requires the use of an arbitrary subspace decomposition in the randomized case. Furthermore, regardless of the chosen subspace, a matrix exponentiation is still required to complete the update step in the case of the SPD manifold, and hence the per-iteration complexity is still $\O(n^3)$. However, \cite{gutman2022coordinate} extend the complexity-reducing subspace selection approach of  \cite{shalit2014coordinate} for the set of orthogonal matrices to the general Stiefel manifold. Our work is similar to \cite{shalit2014coordinate, gutman2022coordinate} in this regard as we propose a complexity-reducing subspace selection for the SPD manifold. Other works, such as \cite{huang21wasserstein,firouzehtarash2021riemannian}, employ similar ideas to those proposed by \cite{gutman2022coordinate} for product manifolds, updating variables block-wise corresponding to a manifold component of the product manifold at a time. 

\section{Notation and Background}\label{paper.sec.2}
This section specifies the notation employed in the work and discusses the relevant background that is necessary to understand the development of the proposed optimization algorithms. 

\subsection{Notation} We denote vectors (matrices) by boldface lower (upper) case letters. The trace and transpose operations are denoted by $\tr{\cdot}$ and $(\cdot)^\T$, respectively. The indicator function $\ind_{\mathcal{C}}$ takes the value 1 when the condition $\mathcal{C}$ is true, and 0 otherwise. The $(i,j)$-th entry, $i$th row and $i$th column of matrix $\A$ is written as $\left[\A\right]_{ij}$, $\left[\A\right]_{i:}$ and $\left[\A\right]_{:i}$ respectively. We denote the $n \times n$ identity matrix by $\I_n$ or by $\I$ when its size is clear from the context. Similarly, $\I_{ij}$ denotes a square matrix whose $(i,j)$-th entry is 1 and all other entries are 0, with the size of the matrix being inferred from the context. The lower triangular Cholesky factor of the symmetric positive definite matrix $\A$ is denoted by $\cL( \A)$, so that $\A = \cL(\A)\cL(\A)^\T$ and its smallest eigenvalue is denoted by $\lambda_{\min}(\A)$. The gradient of a function $f:\Rn^{n\times n} \rightarrow \Rn$ is denoted by $\grd f(\X)$, while its Riemannian gradient is denoted by $\grd^R f(\X)$. We denote Euclidean and Riemannian Hessians by $Hf(\X)$ and $H^Rf(\X)$ respectively. For a Riemnnian manifold $\mathcal{M}$ with Riemannian connection $\nabla$, the Riemannian Hessian $H^Rf(\X)$ is the linear map $H^Rf(\X):T_{\X}\mathcal{M}\rightarrow T_{\X}\mathcal{M}$ defined as $H^Rf(\X)[\V]=\nabla_{\V}\grd^R  f(\X)$ \cite[p.90]{boumal2023introduction}. The Euclidean and Frobenius norms are denoted by $\norm{\cdot}_2$ and  $\norm{\cdot}_F$, respectively. The tangent space at a point $\X \in \Pn$ is denoted by $T_\X\Pn$, with tangent vectors denoted by boldface Greek lower case letters, e.g., $\xib$. Given two tangent vectors $\xib$, $\etab \in T_\X\Pn$, their inner product is given by $\ip{\xib}{\etab}_\X$ and the corresponding norm is given by $\norm{\xib}_\X := \sqrt{\ip{\xib}{\xib}_\X}$.  
Given arbitrary $\X, \Y \in \Pn$, and the geodesic $\gamma(\lambda)$ joining them so that $\gamma(0)=\X$ and $\gamma(1)=\Y$, the tangent vector at $\X$ is denoted by $\xib_{\X\mathbf{Y}} := \gamma'(0)$. The power and exponential maps of an SPD matrix $\W$ with eigenvalue decomposition $\U\D\U^\T$ are given by 
\begin{align}
	\W^k &= \U\D^k\U^\T & \exp(\W) &= \U\exp(\D)\U^\T = \sum_{k=0}^\infty \frac{\W^k}{k!} \label{matexp}
\end{align}
for $k \in \{0,1,2,\dots\}$, where $[\D^k]_{ii} = [\D]_{ii}^k$ and $[\exp(\D)]_{ii} = \exp([\D]_{ii})$ for all $1\leq i \leq n$. 

\subsection{Background on manifold optimization} This work concerns with Riemannian manifolds, which are manifolds equipped with the Riemannian metric (\cite{lee2018introduction}). For the manifold of SPD matrices considered here, the tangent space $T_\X\Pn$ is naturally isomorphic to the set of symmetric matrices $\mathbb{S}^n$ (\cite{bridson2011metric}), and endowed with the Riemannian metric
\begin{eqnarray}
	\ipx{\xib}{\etab} = \tr{\X^{-1}\xib\X^{-1}\etab}. \label{metric} 
\end{eqnarray}
The choice of metric in \eqref{metric} renders $\Pn$ a Hadamard manifold, i.e., a manifold with non-positive section curvature. Hadamard manifolds are useful when designing optimization algorithms since they have a unique distance minimizing curve (geodesic) between any two points. Specifically, the geodesic $\gamma:[0,1]\rightarrow \Pn$ starting at $\X \in \Pn$ in the direction of the tangent vector $\gamma'(0) = \xib$ is given by \cite{bhatia2007positive}:
\begin{align}
	\gamma(\lambda) &= \X^{1/2}\exp\left(\lambda\X^{-1/2}\xib\X^{-1/2}\right)\X^{1/2}.  \label{geodesicsqrt}
\end{align}
Equivalently, given the Cholesky factorization $\X = \B\B^\T$, the geodesic can also be written as
\begin{align}
	\gamma(\lambda) &= \B\exp\left(\lambda\B^{-1}\xib\B^{-\T}\right)\B^{\T}.  \label{geodesiccholesky}
\end{align}
Finally, $\Expx:T_\X\Pn \rightarrow \Pn$ is the exponential map such that $\Expx(\xib) = \gamma(1)$.

For manifold $\Pn$, the parallel transport of a tangent vector $\etab \in T_{\X}\Pn$ along the geodesic in \eqref{geodesicsqrt} is given by \cite{Sra2015conic}:
\begin{eqnarray}
	P_{\X}(\lambda) &=& \X^{1/2} \exp\left( \lambda\frac{1}{2}\X^{-1/2}\xib \X^{-1/2}\right) \X^{-1/2}\eta\X^{-1/2}\exp\left( \lambda\frac{1}{2}\X^{-1/2}\xib \X^{-1/2}\right)\X^{1/2}. \label{parllel_trsprt}
\end{eqnarray}
We remark that the parallel transport is required to calculate the momentum in accelerated gradient algorithms (\cite{Bonnabel2013stoch,zhang2018r,ZhangNIPS2016,kasai2016riemannian,tripuraneni2018averaging,MASAGABabanezhad2019,hosseini2020recent,ahn2020nestero}) and is the most computationally intensive step of these algorithms. The proposed RRSD and RGSD algorithms will however not use the parallel transport step. Next, we introduce some  important definitions for the Riemannian manifold $\Pn$ with the Riemannian metric as specified in \eqref{metric}.

We first look at the definitions of the directional derivative and the Riemannian gradient $\grd^Rf(\X)$ \cite[p.~40]{absil2007optimization} \cite[p.~24]{boumal2014optimization}.

\begin{definition}[\textbf{Directional derivative}]
	Let $f:\Pn \rightarrow \Rn$ be a smooth function and $\gamma(\lambda):\mathbb{R}\rightarrow \Pn$ be a smooth curve satisfying $\gamma(0)=\X$ and $\gamma'(0)=\xib$.  The directional derivative of $f$  at $\X$ in the direction $\xib \in T_{\X}\Pn$ is the
	scalar \cite[p.~40]{absil2007optimization} \cite[p.~24]{boumal2014optimization}:
	\begin{align}
		Df_{\X}(\xib)= \frac{d}{dt}f(\gamma(\lambda))\Big|_{\lambda=0} 
	\end{align}
\end{definition}
\begin{definition}[\textbf{Riemannian gradient}]\label{Riemann_grad}
	The Riemannian gradient of a differentiable function $f:\Pn\rightarrow \Rn$ at $\X \in \Pn$  is defined as the unique tangent vector  $\grd^R f(\X) \in T_{\X}\Pn$ satisfying \cite[p.~46]{absil2007optimization} \cite[p.~26]{boumal2014optimization}:
	\begin{align}
		Df_{\X}(\xib)&=\langle \grd^R f(\X),\xib\rangle_{\X}
	\end{align}
\end{definition}
The Riemannian and Euclidean gradients for $\Pn$ are related by \cite[p. 722]{Sra2015conic} \cite[p. 506]{ferreira2020first}
\begin{align}
	\grd^R f(\X) = \X\grd f(\X)\X. \label{RiemanEuclideangrad}
\end{align} 

We make the following assumption about the smoothness of $f$. 

\begin{assumption}\label{a-smooth}
	The function $f:\Pn \rightarrow \Rn$ is such that its gradient vector field is $L$-Lipschitz smooth.
\end{assumption}
An implication of Assumption \ref{a-smooth} is that given two arbitrary points $\X$, $\Y \in \Pn$, and connecting  geodesic starting at $\X$ with $\gamma'(0) = \xib_{\X\Y}$, the following inequality holds (\cite{zhang16PMLR,LiuNIPS2017}):
\begin{align}
	f(\Y) \leq f(\X) +\langle \grd^R f(\X), \xib_{\X\Y}\rangle_{\X} + \frac{L}{2}\|\xib_{\X\Y}\|_{\X}^2. \label{Lipschitz_cont_grad}
\end{align}
An example of a function with Lipschitz smooth gradient is (\cite{ferreira2019gradient}) 
\begin{align}
	f(\X) = a \log(\det(\X)^{b_1}+b_2)-c\log(\det\X)\label{lips_sooth_exmpl}
\end{align}
for positive numbers $a$, $b_1$, $b_2$, and $c$ with constant $L < ab_1^2n$. Interestingly, gradient vector fields of the functions $\tr{\C\X^{-1}}$  and $\tr{\C\X}$ are not Lipschitz smooth, unless restricted to a compact set.

Next, we will consider the notion of (strong) convexity of functions over $\Pn$. 

\begin{definition}[\textbf{Geodesically convex functions}]\label{paper.def1}
	A function	$f : \Pn\rightarrow \Rn$ is geodesically  convex (g-convex) if its restrictions to all geodesics are convex \cite[p. 64]{rapcsak1997smooth}.
\end{definition}
Definition \ref{paper.def1} implies that, for $0\leq \lambda\leq 1$, the following inequalities hold for every geodesic $\gamma(\lambda)$ joining the two arbitrary points $\X, \Y \in \Pn$:
\begin{align}
	f(\gamma(\lambda)) &\leq (1-\lambda)f(\X)+\lambda f(\Y).
\end{align}
For instance, both $\tr{\C\X^{-1}}$  and $\tr{\C\X}$ are  g-convex functions in $\Pn$ for $\C\succcurlyeq 0$.

\begin{definition}[\textbf{geodesically	$\mu$-strongly convex function}]
	Function $f : \Pn\rightarrow \mathbb{R}$  is called  geodesically
	$\mu$-strongly convex ($\mu$-strongly g-convex) if for any two arbitrary points $\X$, $\Y \in \Pn$, and  connecting  geodesic starting at $\X$ with $\gamma'(0) = \xib_{\X\Y}$, the following inequality holds (\cite{zhang16PMLR,LiuNIPS2017}):
	\begin{align}
		f(\Y)&\geq  f(\X)+\langle \grd^R f(\X), \xib_{\X\Y}\rangle_{\X} + \frac{\mu}{2}\|\xib_{\X\Y}\|_{\X}^2.  \label{geodesic_strong_convx}
	\end{align}
\end{definition}

Having defined the notion of strong convexity in Riemannian spaces, we state the following assumption. 

\begin{assumption}\label{stronglyconvex}
	The function $f:\Pn \rightarrow \Rn$ is $\mu$-strongly g-convex. 
\end{assumption}

As in the Euclidean case, the square of the distance function in $\Pn$, given by
\begin{align}
	d(\X,\Y)^2&=\|\log(\X^{-\frac{1}{2}}\Y\X^{-\frac{1}{2}})\|_F^2,
\end{align}
is $\mu$-strongly g-convex with $\mu=2$. Another example of  $\mu$-strongly g-convex function that also belongs to the function class $\cF$ defined in \eqref{sc_fun_class} is
\begin{align}
	f(\X)=\tr{\C\X^{-1}+\D\X}
\end{align} 
for $\C \succ 0, \D \succ 0$, and with  $\mu=\min\left(\lambda_{\min}(\C),\lambda_{\min}(\D)\right)$ (see Lemma 10). 

The RGD update equation in Riemannian manifold $\Pn$ in the descent direction $-\grd^R f\left(\X\right)$ is given by 
	\begin{align}
			\X_{t+1}
			&= \B_t\exp\left(-\alpha_t\B_t^{-1}\grd^R f\left(\X\right)\B_t^{-\T}\right)\B_t^\T \label{gdupdate}
	\end{align}
	where, $\alpha_{t}$ is step-size and $\B_t$ is Cholesky factor of $\X_t$.

Both the RGD and update matrices of the proposed algorithms depend on the Riemannian gradient $\grd^R f\left(\X\right)$ through the function $\F(\X)=\B^{-1}\grd^R f\left(\X\right)\B^{-\T}$, as we shall see later in  Sec. \ref{sec_rsd_general}. The RGD depends on the entirety of $\F(\X)$, while the update matrices of the proposed algorithms depend only on a few entries $\{[\F(\X)]_{ij}\}_{(i,j)\in \cE}$ for $|\cE|\leq n$. To keep the discussion general, we will assume that calculating the full matrix $\F(\X)$ incurs a complexity of $\O(M)$ while calculating a single entry $[\F(\X)]_{ij}$ incurs a complexity of $\O(m)$, where $m\leq M \leq \frac{mn(n+1)}{2}$. In general, we note that $M = m = \O(n^3)$, unless the function has a special structure.

\section{The Riemannian  Randomized Subspace Descent Algorithm} \label{sec_rsd_general}
In this section, we detail the proposed RRSD algorithm and discuss its computational advantages over the other Riemannian first-order algorithms. We begin with the observation that unlike the coordinate descent algorithm in Euclidean spaces, it is generally not possible to update a specific entry of the iterate $\X_t$ at low complexity. Instead, we must identify a set of orthonormal basis vectors that span the tangent space at $\X_t$ and carry out the updates along one or more of these bases. Interestingly, the selected subspace and the identified basis vectors will be such that they would allow us to directly compute the Cholesky factorization $\B_{t+1}$ of $\X_{t+1}$ in terms of the Cholesky factor $\B_t$ of $\X_t$. As a result, the proposed algorithm will have updates of the form  $\B_{t+1} = \B_t\Bup{t+1}$, where the update matrix $\Bup{t+1}$ is sparse and can be calculated efficiently.

\subsection{Uni-directional update}\label{sec-uni}
We begin with identifying a canonical basis of the tangent space $T_\X\Pn$  at a given $\X$, and show that the updates along such a basis can be carried out directly in terms of the Cholesky factor $\B$ of $\X$. Hence, by updating along a single basis at every iteration, we obtain the so-called Riemannian version of the coordinate descent, whose per-iteration complexity is $\O(m+n)$. Recall that the complexity of calculating a single entry of $\F(\X) = \B^{-1}\grd^R f(\X)\B^{-\mathsf{T}}$ is taken to be $\O(m)$. 

\subsubsection{Canonical Basis} Observe that the tangent space $T_{\X}\Pn$ at $\X$ is spanned by orthonormal basis vectors 
\begin{align}
	\mathbf{G}_{ij}^R(\X)&=\B\mathbf{E}_{ij}\B^\T \label{gij}
\end{align}
for all $1\leq j\leq i \leq n$, where recall that $\B = \cL(\X)$, and 
\begin{align}\label{eij}
	\Eb_{ij} = (\I_{ij}+\I_{ji})\sqrt{2}^{-1-\ind_{i=j}}
\end{align}
where the indicator $\ind_{i= j}$ is 1 when $i = j$ and 0 otherwise. The definition in \eqref{eij} ensures that $\Eb_{ij}$ contains $1/\sqrt{2}$ at the $(i,j)$-th and $(j,i)$-th locations, and 0 elsewhere for $i \neq j$, while $\Eb_{ii}$ contains 1 at the $(i,i)$-th location, and 0 elsewhere. The orthonormality of the basis vectors can be verified by seeing that
\begin{align}
	\ip{\G_{ij}^R(\X)}{\G_{k\ell}^R(\X)}_{\X} &= \tr{\G_{ij}^R(\X)\X^{-1}\G_{k\ell}^R(\X)\X^{-1}} \\
	&= \tr{\B\Eb_{ij}\B^\T\B^{-
			\T}\B^{-1}\B\Eb_{k\ell}\B^\T\B^{-\T}\B^{-1}} \\
	&= \tr{\Eb_{ij}\Eb_{k\ell}} = \ind_{i=k}\ind_{j=\ell}
\end{align}
for all $1\leq j \leq i \leq n$ and $1 \leq \ell \leq k \leq n$. It follows therefore that the Riemannian gradient of $f$ can be written as
\begin{align}
	\grd^R f(\X)&= \sum_{1\leq j \leq i\leq n} \beta_{ij}(\X) \G_{ij}^R(\X) \label{grad-decomp}
\end{align}
where 
\begin{align}
	\beta_{ij}(\X) &= \ip{\grd^R f(\X)}{\G_{ij}^R(\X)}_{\X} \\
	&= \tr{\grd^R f(\X)\X^{-1} \G_{ij}^R(\X)\X^{-1}} \\
	&= \tr{\B^{-1}\grd^R f(\X)\B^{-\T}\Eb_{ij}} \\
	&= \tr{\F(\X)\Eb_{ij}} \\
	&=\sqrt{2}^{\ind_{i\neq j}}[\F(\X)]_{ij}.\label{grad_proj}
\end{align}
Henceforth, we will drop the argument of the coefficient $\beta_{ij}(\X)$, and simply denote it as $\beta_{ij}$ for all $1\leq j \leq i \leq n$. As per our notation, the complexity of calculating $\beta_{ij}$ for any $1\leq j \leq i \leq n$ is $\O(m)$.

\subsubsection{Updating the Cholesky factors}
Having identified the appropriate canonical basis, we can now write down the Riemannian version of randomized coordinate descent, where we only update $\X_t$ along a single randomly selected basis vector. We note that the update along $\G_{ij}^R(\X_t)$ is given by 
\begin{align}
	\X_{t+1} &= \Expxt \left(-\alpha_t \beta_{ij}\G_{ij}^R(\X_t)\right) 
	\label{update1} \\
	&= \B_t \exp(-\alpha_t\beta_{ij}\Eb_{ij})\B_t^\T. \label{update_CD}
\end{align}
The special form of $\Eb_{ij}$ allows us to calculate the Cholesky factor $\cL(\exp(-\alpha_t\beta_{ij}\Eb_{ij}))$ efficiently.  Let us split the subsequent results into two cases: (a) $i \neq j$ and (b) $i = j$. 

\noindent \textbf{Case $i \neq j$. } Denoting $w = \exp(-\alpha_t\beta_{ij}\sqrt{2}^{-\ind_{i\neq j}})$, we note that
\begin{align*}
	\exp(-\alpha_t\beta_{ij}\Eb_{ij}) = \I + (\I_{ii}+\I_{jj})(\frac{w}{2}+\frac{1}{2w}-1) + {(\I_{ij}+\I_{ji})(\frac{w}{2}-\frac{1}{2w})};
\end{align*}
for $i\neq j$. Therefore, we have that
\begin{align}
	\cL(\exp(-\alpha_t\beta_{ij}\Eb_{ij})) = \I + \I_{jj}(u-1) + \I_{ii}(\frac{1}{u}-1) + \I_{ij}\frac{w-1/w}{2u};
\end{align}
where $u = \sqrt{\frac{w+1/w}{2}}$. In other words, the $\cL(\exp(-\alpha_t\beta_{ij}\Eb_{ij}))$ contains ones along the main diagonal, except for the $(j,j)$-th and $(i,i)$-th entries, where it contains $u$ and $1/u$, respectively, and contains $\frac{w-1/w}{2u}$ at the {$(i,j)$-th location}. 

\noindent \textbf{Case $i =j$. } For this case, we have that $\exp(-\alpha_t\beta_{ii}\Eb_{ii}) = \I + \I_{ii}(w-1)$ so that
\begin{align}
	\cL(\exp(-\alpha_t\beta_{ii}\Eb_{ii})) = \I + \I_{ii}(\sqrt{w}-1).
\end{align}

Combining the two cases, we have that 
\begin{align}
	\cL(\exp(-\alpha_t\beta_{ij}\Eb_{ij})) = \I &+ \ind_{i\neq j}\left[\I_{jj}(u-1) + \I_{ii}(\frac{1}{u}-1) + {\I_{ij}}\frac{w-1/w}{2u}\right] \nonumber\\
	&+ \ind_{i=j}\I_{ii}(\sqrt{w}-1) \label{cholesky1}
\end{align}
so that the update can be carried out as
\begin{align}
	\B_{t+1} = \B_t \Bup{t+1} \label{cd-update}
\end{align}
where $\Bup{t+1} = \cL(\exp(-\alpha_t \beta_{ij}\Eb_{ij}))$ for any $1 \leq j \leq i \leq n$, as given in \eqref{cholesky1}. Observe that since $\beta_{ij}$ can be calculated in $\O(m)$ time and the other calculations in \eqref{cholesky1} require $\O(1)$ time, the overall complexity of carrying out the update in \eqref{cd-update} is $\O(m+n)$. 

\subsection{Multi-directional updates}
In this section, we generalize the coordinate descent idea of Sec. \ref{sec-uni} to allow update along multiple randomly selected basis vectors. Unless care is taken however, such updates need not be efficient. Indeed, updating along multiple basis vectors might render $\Bup{t+1}$ dense in general, which would again require $\O(n^3)$ computations per update, and not yield any significant computational advantage. 

In order to ensure that the Cholesky factors can be directly updated using a sparse update matrix $\Bup{t+1}$, we must update along \emph{non-overlapping} bases. Two basis vectors $\G^R_{ij}(\X)$ and $\G^R_{k\ell} (\X)$ are said to be non-overlapping if $i \neq k$, $i \neq \ell$, $j \neq k$, and $j \neq \ell$, or equivalently, $\{i,j\} \cap \{k,\ell\} = \emptyset$ (while letting $\{i,i\} := \{i\}$). For such pairs of basis vectors, it can be seen that  non-zero entries of corresponding $\E_{ij}$ and $\E_{k\ell}$ lie on different rows and columns. As a result, the updates along each such direction can be applied to $\B_t$ in parallel, and the resulting matrices can be added together to obtain $\B_{t+1}$. 

More precisely, for non-overlapping bases $\G^R_{ij}(\X)$ and $\G^R_{k\ell} (\X)$, it holds that
\begin{align}
	\exp(-\alpha_t\beta_{ij}\Eb_{ij} - \alpha_t \beta_{k\ell}\Eb_{k \ell}) = \exp(-\alpha_t\beta_{ij}\Eb_{ij}) + \exp(- \alpha_t \beta_{k\ell}\Eb_{k \ell}) - \I 
\end{align}
implying that
\begin{align}
	\cL(\exp(-\alpha_t\beta_{ij}\Eb_{ij} - \alpha_t \beta_{k\ell}\Eb_{k \ell})) = \cL(\exp(-\alpha_t\beta_{ij}\Eb_{ij})) + \cL(\exp(- \alpha_t \beta_{k\ell}\Eb_{k \ell})) - \I.
\end{align}
Thus, the cumulative update matrix $\Bup{t+1}$ when updating along non-overlapping bases is the sum of individual update matrices corresponding to each basis vector.

In the general case, let $\cE_t := \{i_t^k,j_t^k\}_{k=1}^{K_t}$ be a set of $1\leq K_t \leq n$ pairs of indices, such that $\{i_t^k, j_t^k\} \cap \{i_t^\ell, j_t^\ell\} = \emptyset$ for all $1\leq k,\ell  \leq K_t$. Then, it follows that each pair   $(\G^R_{i_t^kj_t^k},\G^R_{i_t^\ell j_t^\ell})$ is non-overlapping. At the $t$-th iteration, we consider the direction
\begin{align}
	\S(\X_t) = \sum_{(i,j) \in \cE_t} \beta_{ij}\G^R_{ij}(\X_t) = \B_t\Big(\sum_{(i,j)\in\cE_t} \beta_{ij}\Eb_{ij}\Big)\B_t^\T \label{subspace_direction}
\end{align}
which results in the update
\begin{align}
	\X_{t+1} &= \Expxt(-\alpha_t \S(\X_t)) \\
	&= \B_t\exp\Big(-\alpha_t\sum_{(i,j)\in \cE_t}\beta_{ij}\Eb_{ij}\Big)\B_t^\T\label{multi_SD_update}
\end{align}
similar to \eqref{update_CD}. As mentioned earlier, for non-overlapping bases, we have that
\begin{align}
	\exp\Big(-\alpha_t\sum_{(i,j)\in \cE_t}\beta_{ij}\Eb_{ij}\Big) &= \sum_{(i,j)\in \cE_t} \exp(-\alpha_t\beta_{ij}\Eb_{ij}) - (K_t-1)\I \label{exp_sd}\\
	\Rightarrow \cL\Big(\exp\Big(-\alpha_t\sum_{(i,j)\in \cE_t}\beta_{ij}\Eb_{ij}\Big)\Big) &= \sum_{(i,j)\in \cE_t} \cL(\exp(-\alpha_t\beta_{ij}\Eb_{ij})) - (K_t-1)\I =: \Bup{t+1} \label{btup}
\end{align}
and the required multi-directional update is given by 
\begin{align}
	\B_{t+1} = \B_t \Bup{t+1}.\label{md-update}
\end{align}

We observe that the unidirectional update in \eqref{cd-update} corresponds to $K_t = 1$ in \eqref{md-update}. In general, it is always possible to choose non-overlapping bases such that $K_t \geq \lfloor\frac{n}{2}\rfloor$. On the one extreme, we can select $\cE_t$ such that $i_t^k \neq j_t^k$ for all $1 \leq k \leq K_t$, so that $K_t = \lfloor\frac{n}{2}\rfloor$. On the other extreme, if $\cE_t$ is such that $i_t^k = j_t^k$ for all $1 \leq k \leq K_t$, then we have that $K_t = n$.

The update in \eqref{md-update} requires us to calculate $K_t$ entries of $\F(\X)$, thus incurring a complexity of $\O( mn)$ for $K_t = \O(n)$. Next, since each summand in \eqref{btup} can be calculated in $\O(1)$ time, $\B_{t+1}^\text{up}$ can be calculated in $\O(n)$ time and has $\O(n)$ non-zero entries. Finally, the calculation in \eqref{md-update} requires $\O(n^2)$ time, resulting in the overall complexity of $\O(mn+n^2)$. We note these complexity bounds pertain to the case when the basis vectors are chosen randomly or arbitrarily. However, other choices of basis vectors are also possible, and in particular, a greedy approach will be discussed in the next section. 

\section{Riemannian Subspace Descent Variants for $\cF$}\label{RSD_for_function_class}
As established in Sec. \ref{sec_rsd_general}, the per-iteration complexity of the proposed algorithm is $\O(m+n)$ for the uni-directional updates and $\O(mn+n^2)$ for the multi-directional updates. In this section, we show that for the function class $\cF$ in \eqref{sc_fun_class}, it is possible to maintain intermediate variables, so as to allow us to calculate the entries of  $\F(\X)$ efficiently with $m = \O(n)$. As a result, the worst-case per-iteration complexity of the proposed RRSD algoithm when applied to functions in $\cF$ becomes $\O(n)$ in the uni-directional case and $\O(n^2)$ in the multi-directional case. 

We begin with analyzing some of the properties of functions in $\cF$. For brevity, we assume all matrices to be symmetric in the rest of the paper. The gradient of $f \in \cF$ is given by  
\begin{align}
	\grd f(\X)&=-\sum_{p=1}^{P}\frac{d g}{d g_{1,p}}\X^{-1}\C_p\X^{-1}+ \sum_{q=1}^{Q}\frac{d g}{d g_{2,q}}\D_q  +\frac{d g}{dg_3}\X^{-1} \nonumber\\ 
	&+ \sum_{r=1}^{R}\frac{d g}{d g_{4,r}}\left(\A_r\X\H_r+\H_r\X\A_r\right)
	-\sum_{s=1}^{S}\frac{d g}{d g_{5,s}}\left(\X^{-1}\F_s\X^{-1}\G_s\X^{-1}+\X^{-1}\G_s\X^{-1}\F_s\X^{-1}\right) \nonumber\\
	&+\frac{1}{2}\sum_{m=1}^{M}\frac{d g}{d g_{6,m}}\left(\Q_m\X^{-1}\P_m + \P_m\X^{-1}\Q_m-\X^{-1}\P_m\X\Q_m\X^{-1} -\X^{-1}\Q_m\X\P_m\X^{-1} \right)
\end{align}
where,	$g_{1,p}=\tr{\C_p\X^{-1}}$, $g_{2,q}= \tr{\D_q \X}$, $g_3=\log\det \X$, $g_{4,r}=\tr{\A_r\X\H_r\X}$, $g_{5,s}= \tr{\X^{-1}\F_s\X^{-1}\G_s}$ and $g_{6,m}= \tr{\P_m\X\Q_m\X^{-1}}$. Substituting in the expression for $\beta_{ij}$, we obtain
\begin{align}
	\beta_{ij}&= \sqrt{2}^{\ind_{i\neq j}}[\F(\X)]_{ij}=\tr{\grd f(\X)\B\mathbf{E}_{ij}\B^\T}\\
	&= -\sum_{p=1}^{P}\frac{d g}{d g_{1,p}}\tr{\X^{-1}\C_p\X^{-1}\B\mathbf{E}_{ij}\B^\T }+  \sum_{q=1}^{Q}\frac{d g}{d g_{2,q}}\tr{\D_q\B\mathbf{E}_{ij}\B^\T}+ \frac{d g}{dg_3}\tr{\X^{-1}\B\mathbf{E}_{ij}\B^\T}\nonumber\\
	&+\sum_{r=1}^{R}\frac{d g}{d g_{4,r}}\tr{\left[\A_r\X\H_r+\H_r\X\A_r\right]\B\mathbf{E}_{ij}\B^\T}\nonumber\\
	& - \sum_{s=1}^{S}\frac{d g}{d g_{5,s}}\tr{\left[\X^{-1}\F_s\X^{-1}\G_s\X^{-1}+\X^{-1}\G_s\X^{-1}\F_s\X^{-1}\right]\B\mathbf{E}_{ij}\B^\T}
	\nonumber
	\\
	&+\frac{1}{2}\sum_{m=1}^{M}\frac{d g}{d g_{6,m}}\tr{\left(\begin{array}{ll}\Q_m\X^{-1}\P_m -\X^{-1}\P_m\X\Q_m\X^{-1}\\ + \P_m\X^{-1}\Q_m-\X^{-1}\Q_m\X\P_m\X^{-1}
		\end{array} \right)\B\mathbf{E}_{ij}\B^\T}
	\nonumber
	\\
	&= -\sum_{p=1}^{P}\frac{d g}{d g_{1,p}}\tr{\B^{-1}\C_p\B^{-\T}\mathbf{E}_{ij}}+  \sum_{q=1}^{Q}\frac{d g}{d g_{2,q}}\tr{\B^\T\D_q\B\mathbf{E}_{ij}}+ \frac{d g}{dg_3}\tr{\mathbf{E}_{ij}}
	\nonumber\\
		&+\sum_{r=1}^{R}\frac{d g}{d g_{4,r}}\left(\tr{\B^{\T}\A_r \B\B^{\T}\H_r \B\E_{ij}}+\tr{\B^{\T}\H_r \B\B^{\T}\A_r \B\E_{ij}}\right)
	\nonumber
	\\
	& - \sum_{s=1}^{S}\frac{d g}{d g_{5,s}}\left(\tr{\B^{-1}\F_s\B^{-\T}\B^{-1}\G_s\B^{-\T}\E_{ij}}+\tr{\B^{-1}\G_s\B^{-\T}\B^{-1}\F_s\B^{-\T}\E_{ij}}\right)
	\nonumber
	\\
	&+\frac{1}{2}\sum_{m=1}^{M}\frac{d g}{d g_{6,m}}\left(\begin{array}{ll}\tr{\B^{\T}\Q_m\B^{-\T}\B^{-1}\P_m\B\E_{ij}} -\tr{\B^{-1}\P_m\B\B^{\T}\Q_m\B^{-\T}\E_{ij}}\\ + \tr{\B^{\T}\P_m\B^{-\T}\B^{-1}\Q_m\B\E_{ij}}-\tr{\B^{-1}\Q_m\B\B^{\T}\P_m\B^{-\T}\E_{ij}}
			\end{array}\right)
	\nonumber
	\\
	&= -\sqrt{2}^{\ind_{i\neq j}}\sum_{p=1}^{P}\frac{d g}{d g_{1,p}}[\B^{-1}\C_p\B^{-\T}]_{ij}+  \sqrt{2}^{\ind_{i\neq j}}\sum_{q=1}^{Q}\frac{d g}{d g_{2,q}}[\B^\T\D_q\B]_{ij}+ \frac{d g}{dg_3}\ind_{i= j}
	\nonumber\\
	&+\sqrt{2}^{\ind_{i\neq j}}\sum_{r=1}^{R}\frac{d g}{d g_{4,r}}\left(\left[\B_t^{\T}\A_r\B\B^{\T}\H_r\B\right]_{ij}+\left[\B_t^{\T}\A_r\B\B^{\T}\H_r\B\right]_{ji}\right)
	\nonumber\\
	&-\sqrt{2}^{\ind_{i\neq j}}\sum_{s=1}^{S}\frac{d g}{d g_{5,s}}\left(\left[\B^{-1}\F_s\B^{-\T}\B^{-1}\G_s\B^{-\T}\right]_{ij}+\left[\B^{-1}\F_s\B^{-\T}\B^{-1}\G_s\B^{-\T}\right]_{ji}\right)
	\nonumber\\
	&\frac{1}{\sqrt{2}}\sum_{m=1}^{M}\frac{d g}{d g_{6,m}}\left( \begin{array}{ll}
		\left[\B^{\T}\Q_m\B^{-\T}\B^{-1}\P_m\B\right]_{ij}+\left[\B^{\T}\Q_m\B^{-\T}\B^{-1}\P_m\B\right]_{ji}\\
		-\left[\B^{-1}\P_m\B\B^{\T}\Q_m\B^{-\T}\right]_{ij}-\left[\B^{-1}\P_m\B\B^{\T}\Q_m\B^{-\T}\right]_{ji}
	\end{array}\right)
	\label{beta_in_B}
\end{align}
Here, observe that $\beta_{ij}$ depends on $\X_t$ through 
\begin{subequations}\label{M12}
	\begin{align}
		\M_{1,p}(\X_t)&=\B_t^{-1}\C_p \B_t^{-\T} & \M_{2,q}(\X_t)&=\B_t^{\T}\D_q\B_t\\
		\M_{4,1,r}(\X_t)&= \B_t^{\T}\A_r \B_t &	\M_{4,2,r}(\X_t)&= \B_t^{\T}\H_r \B_t\\
		\M_{5,1,s}(\X_t)&= \B_t^{-1}\F_s\B_t^{-\T} &\M_{5,2,s}(\X_t)&= \B_t^{-1}\G_s\B_t^{-\T}\\
		\M_{6,1,m}(\X)&=\B^{-1}\P_m\B & \M_{6,2,m} (\X)  &= \B^{-1}\Q_m\B
	\end{align} 
\end{subequations}
as well as through $g_{1,p}$, $g_{2,q}$, $g_3$, $g_{4,r}$, $g_{5,s}$ and $g_{6,m}$ which in turn, can be calculated as
\begin{align}
	g_{1,p}(\X_t)&=\tr{\C_p\X_t^{-1}}=\tr{\B_t^{-1}\C_p\B_t^{-\T}}=\tr{\M_{1,p}(\X_t)}\\
	g_{2,q}(\X_t)&=\tr{\D_q\X_t}=\tr{\B_t^{\T}\D_q\B_t}=\tr{\M_{2,q}(\X_t)}\\
	g_3(\X_t)&=\log\det\X_t=2\log\det\left(\B_{t}\right)= 2\log\det\left(\B_{t-1}\right)+	2\log\det\left(\Bup{t}\right)\\
	g_{4,r}(\X_t)&=\tr{\X\A_r\X\H_r}=\tr{\B_t^{\T}\A_r \B_t \B_t^{\T}\H_r \B_t}=\tr{\M_{4,1,r}(\X_t)\M_{4,2,r}(\X_t)}\\
	g_{5,s}(\X_t)&=\tr{\X^{-1}\F_s\X^{-1}\G_s} = \tr{\B_t^{-1}\F_s\B_t^{-\T}\B_t^{-1}\G_s\B_t^{-\T}} = \tr{\M_{5,1,s}(\X_t)\M_{5,2,s}(\X_t)}\\
	g_{6,m}(\X)&= \tr{\P_m\X\Q_m\X^{-1}}=\tr{\B^{-1}\P_m\B\B^{\T}\Q_m\B^{-\T}}=\tr{\M_{6,1,m}(\X)\left[\M_{6,2,m} (\X)\right]^{\T}}
\end{align}
Since $\B$ and $\Bup{t}$ are both lower triangular matrices, their determinants can be calculated in $\O(n)$ time. Further, we can maintain $\{\M_{1,p}(\X_t)\}_{1\leq p\leq P}$,  $\left\{\M_{4,1,r}(\X_t),\;\M_{4,2,r}(\X_t)\right\}_{1\leq r\leq R}$, $\{\M_{2,q}(\X_t)\}_{1\leq q\leq Q}$, $\{\M_{5,1,s}(\X_t),\;\M_{5,2,s}(\X_t)\}_{1\leq s\leq S}$ and $\{\M_{6,1,m}(\X_t),\;\M_{6,2,m}(\X_t)\}_{1\leq m\leq M}$ by observing that
\begin{align}
	\M_{1,p}(\X_{t+1})&=\B_{t+1}^{-1}\C_p\B_{t+1}^{-\T}
	= \left[\Bup{t+1}\right]^{-1}\B_{t}^{-1}\C_p\B_{t}^{-\T}\left[\Bup{t+1}\right]^{-\T}\nonumber\\
	&= \left[\Bup{t+1}\right]^{-1}\M_{1,p}(\X_t)\left[\Bup{t+1}\right]^{-\T} \label{m1_update}\\
	\M_{2,q}(\X_{t+1})&=\B_{t+1}^{\T}\D_q \B_{t+1}
	=\left[\Bup{t+1}\right]^\T\B_{t}^{\T}\D_q \B_{t}\Bup{t+1}= \left[\Bup{t+1}\right]^\T\M_{2,q}(\X_t)\Bup{t+1} \label{m2_update}\\
	\M_{4,1,r}(\X_t)&= \B_{t+1}^{\T}\A_r \B_{t+1}	=\left[\Bup{t+1}\right]^\T\B_{t}^{\T}\A_r \B_{t}\Bup{t+1}= \left[\Bup{t+1}\right]^\T\M_{4,1,r}(\X_t)\Bup{t+1}\\
	\M_{4,2,r}(\X_t)&= \B_{t+1}^{\T}\H_r \B_{t+1}	=\left[\Bup{t+1}\right]^\T\B_{t}^{\T}\H_r \B_{t}\Bup{t+1}= \left[\Bup{t+1}\right]^\T\M_{4,2,r}(\X_t)\Bup{t+1}\\
	\M_{5,1,s}(\X_{t+1})&=\B_{t+1}^{-1}\G_s\B_{t+1}^{-\T}
	= \left[\Bup{t+1}\right]^{-1}\B_{t}^{-1}\G_s\B_{t}^{-\T}\left[\Bup{t+1}\right]^{-\T}\nonumber\\
	&= \left[\Bup{t+1}\right]^{-1}\M_{5,1,s}(\X_t)\left[\Bup{t+1}\right]^{-\T}\\
		\M_{5,2,s}(\X_{t+1})&=\B_{t+1}^{-1}\H_s\B_{t+1}^{-\T}
	= \left[\Bup{t+1}\right]^{-1}\B_{t}^{-1}\H_s\B_{t}^{-\T}\left[\Bup{t+1}\right]^{-\T}\nonumber\\
	&= \left[\Bup{t+1}\right]^{-1}\M_{5,2,s}(\X_t)\left[\Bup{t+1}\right]^{-\T}\\
	\M_{6,1,m}(\X_{t+1})&=\B_{t+1}^{-1}\P_m\B_{t+1}
	= \left[\Bup{t+1}\right]^{-1}\B_{t}^{-1}\P_m\B_{t}\Bup{t+1}= \left[\Bup{t+1}\right]^{-1}\M_{6,1,m}(\X_t)\Bup{t+1}\\
\M_{6,2,m}(\X_{t+1})&=\B_{t+1}^{-1}\Q_m\B_{t+1}
	= \left[\Bup{t+1}\right]^{-1}\B_{t}^{-1}\Q_m\B_{t}\Bup{t+1}= \left[\Bup{t+1}\right]^{-1}\M_{6,1,m}(\X_t)\Bup{t+1}\label{m6_update}
\end{align}
where $\Bup{t+1}$ is a sparse matrix that depends on the basis vectors used for the update. The inverse of $\Bup{t+1}$ can be obtained by inverting the diagonal entries and replacing off-diagonal entries by negative of corresponding values of  $\Bup{t+1}$, i.e.,
\begin{align}
	[[\Bup{t+1}]^{-1}]_{ij}&=\frac{\ind_{i=j}}{[\Bup{t+1}]_{ij}}- \ind_{i\neq j}[\Bup{t+1}]_{ij}
\end{align} 
for all $1\leq j \leq i \leq n$. The overall complexity of the updates depends on the number of update directions $K_t$. 

\subsection{Uni-directional updates}\label{rrsdfuni}
For  uni-directional update case, at most two rows and two columns of intermediate variables $\M_{1,p}(\X_t)$,  $\M_{2,q}(\X_t)$, $\M_{4,1,r}(\X_t)$, $\M_{4,2,r}(\X_t)$,  $\M_{5,1,s}(\X_t)$, $\M_{5,2,s}(\X_t)$, $\M_{6,1,m}(\X_t)$, $\M_{6,2,m}(\X_t)$, and $\B_t$ are updated at each iteration, and hence the complexity of the updates in \eqref{m1_update}-\eqref{m6_update} is $\O(n)$. 
Given the variables $\M_{1,p}(\X_t)$, $\M_{2,q}(\X_t)$, functions $f_{1,p}(\X_t)$, $f_{2,q}(\X_t)$, and $f_3(\X_t)$ can be computed in $O(n)$. An interesting observation is that, given the intermediate variables $\M_{4,1,r}(\X_t)$, $\;\M_{4,2,r}(\X_t)$,  $\M_{5,1,s}(\X_t)$, $\M_{5,2,s}(\X_t)$, $\M_{6,1,m}(\X_t)$ and $\M_{6,2,m}(\X_t)$, the calculations of $g_{4,r}(\X_t)$, $g_{5,s}(\X_t)$ and $g_{6,m}(\X_t)$ can also be accomplished in $O(n)$ by adjusting the contribution of modified $O(n)$ entries in the function calculations. Consequently, it can be deduced that $\beta_{ij}$ can be computed in $\O(n)$ time.
Finally, the update in \eqref{cd-update} requires $\O(n)$ calculations, and hence the overall complexity of uni-directional update for $f \in \cF$ is $\O(n)$, which is significantly better than the $\O(n^3)$ complexity of RGD.

We remark that all the variations of the proposed algorithm as well as that of the state-of-the-art gradient descent algorithm and its variants require a memory complexity of $\O(n^2)$. Interestingly however, in the uni-directional case, only $\O(n)$ entries of $\B_{t+1}$ and other intermediate variables are updated at every iteration. Therefore it is possible for us to carry out a single update for the uni-directional case while using only $\O(n)$ space in the random access memory (RAM). 

\subsection{Multi-directional updates}
For the multi-directional case, it can be seen that $\O(n)$ rows and columns of $\M_{1,p}(\X_t)$,  $\M_{2,q}(\X_t)$, $\M_{4,1,r}(\X_t)$, $\M_{4,2,r}(\X_t)$,  $\M_{5,1,s}(\X_t)$, $\M_{5,2,s}(\X_t)$, $\M_{6,1,m}(\X_t)$, $\M_{6,2,m}(\X_t)$, and $\B_t$ are updated at every iteration. Therefore the complexity of updates in \eqref{m1_update}-\eqref{m6_update} is $\O(n^2)$. Consequently, the computational complexity for calculating $\beta_{ij}$ also becomes $\O(n^2)$. Finally, the complexity of the update calculation in \eqref{md-update} for this case is $\O(n^2)$, resulting in an overall complexity of $\O(n^2)$.

\subsection{Greedy Subspace Descent}\label{subsec_greedy}
Thus far, all the versions of the proposed Riemannian subspace descent framework utilize randomly selected basis vectors, analogous to the randomized coordinate descent algorithms (\cite{Nesterov2012cd}). We remark that it is also possible to select the basis vectors using some criteria or heuristic. For instance, one could select the basis vectors that correspond to the largest values of $\abs{\beta_{ij}}$, which correspond to the directions of steepest descent. However, in order to select the index $(i,j)$ corresponding to the largest value of $\abs{\beta_{ij}}$, one must calculate all the entries of $\F(\X_t)$ for each $t$. Hence, the complexity of a uni-directional update increases from $\O(m+n)$ to $\O(M+n)$, which is significant even when $f \in \cF$ with $M = \O(n^2)$. Therefore, the proposed heuristic does not offer any computational advantages when using uni-directional updates. 

In the multi-directional case, selecting the non-overlapping basis vectors so as to ensure that the largest values of $\abs{\beta_{ij}}$ are selected is even more challenging. To this end, we propose a greedy heuristic, where the basis vectors are selected as follows. To begin with, we sort the entries of the lower triangular part of $\F(\X)$ in decreasing order of their absolute values and store the resulting ordering. At the first step, the largest entry is selected, and the overlapping entries are removed. The process is repeated with the remaining entries till all the entries are exhausted and all the rows and columns have been removed. It can be seen that the greedy heuristic ends up selecting $\lfloor n/2 \rfloor \leq K_t \leq n$ entries of $\F(\X)$ which correspond to non-overlapping basis vectors. 

The computational complexity of the proposed RGSD algorithm is not much higher than that of RRSD. In general, calculating the full matrix $\F(\X)$ requires $\O(M)$ time, the resulting entries can be sorted in terms of their absolute values in $\O(n^2\log(n))$ time, and subsequently, the greedy heuristic for selecting non-overlapping bases requires further $\O(n^2)$ time. Hence, the overall complexity of RGSD is $\O(M+n^2\log(n))$, which becomes $\O(n^2\log(n))$ for the case when $f \in \cF$. 

Table \ref{peritercomplexity} summarizes the per-iteration time complexity of gradient descent, and the proposed subspace descent algorithms. 

\begin{table}[h!]
	\begin{center}
		\begin{tabular}{|p{6cm}|p{2cm}|p{2cm}|p{3.5cm}|}
			\hline
			Algorithm &	$M$&  $m$ & Complexity/Iteration\\
			\hline
			RGD &$\O(n^3)$ & $\O(n^3)$ & $\O(n^3)$\\
			\hline
			RRSD (uni-directional, general)	&$\O(n^3)$ & $\O(n^3)$ & $\O(n^3)$\\
			\hline
			RRSD (multi-directional, general) &$\O(n^3)$ & $\O(n^3)$ & $\O(n^3)$\\
			\hline
			RRSD (uni-directional, $f \in \cF$)	&$\O(n^2)$ & $\O(n)$ & $\O(n)$\\
			\hline
			RGSD (uni-directional, $f \in \cF$)	&$\O(n^2)$ & $\O(n)$ & $\O(n^2)$\\
			\hline
			RRSD (multi-directional, $f \in \cF$) &$\O(n^2)$ & $\O(n)$ & $\O(n^2)$\\				
			\hline	
			RGSD (multi-directional, $f \in \cF$)	&$\O(n^2)$ & $\O(n)$ & $\O(n^2\log n)$\\
			\hline
		\end{tabular}
	\end{center}
	\caption{Comparison of  per-iteration time complexity of algorithms}\label{peritercomplexity}
\end{table}

\subsection{Riemannian Randomized subspace descent for finite sum problems}
Finally we comment on the suitability of proposed RRSD algorithm for finite-sum problems as compared to RSGD algorithm \eqref{Riemann_SGD}. Let us consider following finite sum problem 
\begin{align}
	\min_{\X \in \Pn} f(\X)&:=\sum_{s=1}^{S}f_s(\X); & f_s(\X)\in \cF
\end{align}  
Let us assume that each $f_s$ has following functional form:
\begin{align}
	f_s(\X)=g\left(\begin{array}{ll}\tr{\C_s\X^{-1}},\tr{\D_s \X},\log\det\X,\tr{\X\A_s\X\H_s}, \tr{\X^{-1}\F_s\X^{-1}\G_s},
		\\
		\tr{\P_s\X\Q_s\X^{-1}}
	\end{array}\right)
\end{align} 

For the implementation of RSGD \eqref{Riemann_SGD}, we select  an index  $s$ uniformly at random from the index set $\left\{1,2,\dots,S\right\}$ and take a step along the negative of Riemannian gradient of the function $f_s(\X)$. The Riemannian gradient of $f_s(\X)$ is given by 
\begin{align}
	\grd f(\X)&=-\frac{d g}{d g_{1}}\X^{-1}\C_s\X^{-1}+ \frac{d g}{d g_{2}}\D_s +  +\frac{d g}{dg_3}\X^{-1}+\frac{d g}{d g_{4}}\left(\A_s\X\H_s+\H_s\X\A_s\right)\nonumber\\
	&-\frac{d g}{d g_{5}}\left(\X^{-1}\F_s\X^{-1}\G_s\X^{-1}+\X^{-1}\G_s\X^{-1}\F_s\X^{-1}\right)\nonumber\\
	&+\frac{1}{2}\frac{d g}{d g_{6}}\left(\Q_s\X^{-1}\P_s + \P_s\X^{-1}\Q_s-\X^{-1}\P_s\X\Q_s\X^{-1} -\X^{-1}\Q_s\X\P_s\X^{-1} \right)
\end{align}
where $dg/dg_i$ denotes the derivative of $g$ with respect to its $i$-th argument. Finally the update can be written as 
\begin{align}
	\X_{t+1}&= \B_t\exp\left(\alpha_t\left[\begin{array}{ll}\frac{d g}{d g_1}\B_t^{-1} \C_s\B_t^{-\T}-\frac{d g}{d g_2}\B_t^{\T}\D_s\B_t- \frac{d g}{d g_3}\I
		\\
		-\frac{d g}{d g_{4}}\left(\B_t^{\T}\A_r\B\B^{\T}\H_r\B+\B_t^{\T}\H_r\B\B^{\T}\A_r\B\right)\\
		+\frac{d g}{d g_{5}}\left(\B^{-1}\F_s\B^{-\T}\B^{-1}\G_s\B^{-\T}+\B^{-1}\G_s\B^{-\T}\B^{-1}\F_s\B^{-\T}\right)\\
		-\frac{1}{2}\frac{d g}{d g_{6}}\left(\begin{array}{ll}
			\B^{\T}\Q_s\B^{-\T}\B^{-1}\P_s\B +\B^{\T}\P_s\B^{-\T}\B^{-1}\Q_s\B\\
			-\B^{-1}\P_s\B\B^{\T}\Q_s\B^{-\T}- \B^{-1}\Q_s\B\B^{\T}\P_s\B^{-\T}
		\end{array}\right)
	\end{array}\right]\right)\B_t^\T
\end{align}
which incurs $\O(n^3)$ time complexity and requires the inversion of matrix $\B_t$. 

In contrast, for the RRSD algorithm, we have that
\begin{align}
	\beta_{ij}(\X_t)
	&= \sum_{s=1}^{S}\left(\begin{array}{ll}
		-\sqrt{2}^{\ind_{i\neq j}}\frac{d g}{d g_1}[\B_t^{-1} \C_s\B_t^{-\T}]_{ij}+\sqrt{2}^{\ind_{i\neq j}}\frac{d g}{d g_2}[\B^{\T}\D_s\B]_{ij}+\frac{d g}{d g_3}\ind_{i\neq j}\\
		+ \sqrt{2}^{\ind_{i\neq j}}\frac{d g}{d g_{4,r}}\left(\left[\B_t^{\T}\A_s\B\B^{\T}\H_s\B\right]_{ij}+\left[\B_t^{\T}\A_s\B\B^{\T}\H_s\B\right]_{ji}\right)\\
		-\sqrt{2}^{\ind_{i\neq j}}\frac{d g}{d g_{5}}\left(\left[\B^{-1}\F_s\B^{-\T}\B^{-1}\G_s\B^{-\T}\right]_{ij}+\left[\B^{-1}\F_s\B^{-\T}\B^{-1}\G_s\B^{-\T}\right]_{ji}\right)\\
		+\frac{1}{\sqrt{2}^{\ind_{i\neq j}}}\frac{d g}{d g_{6}}\left(\begin{array}{ll}\left[\B^{\T}\Q_s\B^{-\T}\B^{-1}\P_s\B\right]_{ij}+\left[\B^{\T}\Q_s\B^{-\T}\B^{-1}\P_s\B\right]_{ji}\\
		-\left[\B^{-1}\P_s\B\B^{\T}\Q_s\B^{-\T}\right]_{ij}-\left[\B^{-1}\P_s\B\B^{\T}\Q_s\B^{-\T}\right]_{ji}
		\end{array}\right)
	\end{array}\right) \label{beta_SGD}
\end{align}
Recalling the definitions of $\M_{1,s}(\X_t)$,$\M_{2,s}(\X_t)$,$\M_{4,1,s}(\X_t)$,$\M_{4,2,s}(\X_t)$,$\M_{5,1,s}(\X_t)$,$\M_{5,2,s}(\X_t)$, $\M_{6,1,s}(\X_t)$, and $\M_{6,2,s}(\X_t)$ from \eqref{M12}, we can write down the recursive update rules similar to those in \eqref{m1_update}-\eqref{m6_update}. Further, $\Bup{t+1}$ can be calculated as explained in the Sec. \ref{sec_rsd_general}, resulting in the update $\B_{t+1} = \B_t \Bup{t+1}$. For the uni-directional and multi-directional cases, the total complexities of maintaining these intermediate variables and updating $\B_{t+1}$ is $\O(Sn)$ and $\O(Sn^2)$ respectively, at every iteration. Therefore, compared to RSGD, the proposed uni-directional RRSD incurs a lower per-iteration complexity if $S = o(n^2)$ while multi-directional RRSD incurs a lower complexity for $S = o(n)$. It is further remarked that for the uni-directional updates, only two rows and two columns of $\{\M_{1,s}(\X_t),  \M_{2,s}(\X_t), \newline \M_{4,1,s}(\X_t), \M_{4,2,s}(\X_t),  \M_{5,1,s}(\X_t), \M_{5,2,s}(\X_t), \M_{6,1,s}(\X_t), \M_{6,2,s}(\X_t)\}_{s=1}^S$ are updated at every iteration, and hence the update requires only $\O(Sn)$ space in the RAM. 

\section{Convergence analysis} 
In  this section, we characterize the iteration-complexity of the proposed algorithms required to achieve $\epsilon$-accuracy. In general, $\X_t$ is said to be $\epsilon$-accurate if it satisfies $f(\X_t) \leq f(\X^{\star}) + \epsilon$. However, for the randomized variants of the proposed algorithm, we only require that $\Ex{f(\X_t)}\leq f(\X^{\star}) + \epsilon$. We show that under Assumptions \ref{a-smooth} and \ref{stronglyconvex}, all the proposed algorithms exhibit linear convergence. However, the uni-directional variants of RRSD and RGSD generally require $\O(n)$ times as many iterations as those required by their multi-directional counterparts. The subsequent bounds will depend on the condition number $\kappa := L/\mu$ and the initialization through $D_0:=f(\X_0)-f(\X^{\star})$. For all the algorithms, we set the step size $\alpha = 1/L$. 

\subsection{Uni-directional update}
\begin{theorem}\label{rrsdic}
	Under Assumptions \ref{a-smooth} and \ref{stronglyconvex}, the uni-directional RRSD algorithm has an iteration complexity of $\O(n^2\kappa \log(\frac{D_0}{\epsilon}))$. 
\end{theorem}
\begin{proof}
	At the $t$-th iteration, $\G^R_{i_tj_t}(\X_t)$  is selected uniformly at random from among $d=\frac{n(n+1)}{2}$ possible canonical basis directions.
	Recalling the definition of $\xib_{\X\Y}$ in Sec. \ref{paper.sec.2}, we see from the update in \eqref{update1} that $\xib_{\X_t\X_{t+1}}= -\frac{1}{L}\beta_{i_tj_t}\G_{i_tj_t}^R(\X_t)$. Therefore, we have from \eqref{Lipschitz_cont_grad} that
	\begin{align}
		f\left(\X_{t+1}\right)&\leq 	f\left(\X_t\right)-\frac{\beta_{i_tj_t}^2}{L}+ \frac{\beta_{i_tj_t}^2}{L^2}\frac{L}{2}
		\;=\;	f\left(\X_t\right)-\frac{\beta_{i_tj_t}^2}{2L} \nonumber
	\end{align}
	Since $(i_t,j_t)$ are independent identically distributed, taking conditional expectation given $\{\X_\tau\}_{\tau=1}^t$, we have that
	\begin{align}
		\Et{f\left(\X_{t+1}\right)} &\leq  f(\X_t) - \frac{1}{2dL}\sum_{1\leq j \leq i \leq n} \beta_{ij}^2(\X_t) \\
		&= f\left(\X_t\right)-\frac{1}{2dL}\|\grd^R f(\X_t)\|_{\X_t}^2\label{exp_at_t}
	\end{align}
	where $\Et{\cdot}$ denotes the expectation with respect to $(i_t,j_t)$, and \eqref{exp_at_t} follows from the orthonormal decomposition of $\grd^R f(\X_t)$ in \eqref{grad-decomp}. Substituting $\Y=\X^{\star}$ in  \eqref{geodesic_strong_convx} and rearranging, we obtain
	\begin{align}
		\langle \grd^R f(\X), -\xib_{\mathbf{XX}^{\star}}\rangle_{\X} &\geq  f(\X)-f(\X^{\star})+ \frac{\mu}{2}\|\xib_{\mathbf{XX}^{\star}}\|_{\X}^2. \label{grad_bound_intermediate}
	\end{align}
	Since $f(\X)-f(\X^{\star})\geq 0$, we have from the Cauchy-Schwarz inequality that
	\begin{align}
		\|\grd^R f(\X)\|_{\X}\|\xib_{\mathbf{XX}^{\star}}\|_{\X} &\geq  \frac{\mu}{2}\|\xib_{\mathbf{XX}^{\star}}\|_{\X}^2
		\nonumber
	\end{align}
	which implies that
	\begin{align}
		\|\grd^R f(\X)\|_{\X} \geq  \frac{\mu}{2}\|\xib_{\mathbf{XX}^{\star}}\|_{\X} \label{grad_bound}
	\end{align}
	since $\X \neq \X^\star$. Again, from \eqref{grad_bound_intermediate}, we have that
	\begin{align}
		\Rightarrow   \|\grd^R f(\X)\|_{\X}\|\xib_{\mathbf{XX}^{\star}}\|_{\X} &\geq f(\X)-f(\X^{\star}) \label{function_bound}
	\end{align}
	which, together with \eqref{grad_bound} yields
	\begin{align}
		{\frac{2}{\mu}\|\grd^R f(\X)\|_{\X}^2 \geq  f(\X)-f(\X^{\star})}\label{grad_dominate}
	\end{align}
	Combining the bounds in \eqref{exp_at_t} and \eqref{grad_dominate}, we have that
	\begin{align}
		\Et{f\left(\X_{t+1}\right)}&\leq  f\left(\X_t\right)-\frac{1}{2dL}\frac{\mu}{2}\left(f(\X_t)-f(\X^{\star})\right) \nonumber
		\\
		\Rightarrow \Et{f\left(\X_{t+1}\right)}-f(\X^{\star})&\leq  f\left(\X_t\right)-f(\X^{\star})-\frac{\mu}{4d L}\left(f(\X_t)-f(\X^{\star})\right) \\
		&=  \left(1-\frac{\mu}{4d L}\right)\big[f(\X_t)-f(\X^{\star})\big].
	\end{align}
	Taking full expectation on both sides and continuing the recursion, we obtain the desired linear convergence bound
	\begin{align}
		\Ex{f\left(\X_{t+1}\right)}-f(\X^{\star})\leq  \left(1-\frac{\mu}{4d L}\right)^{t+1}\big[f(\X_0)-f(\X^{\star})\big]\label{converge_uni}.
	\end{align}
	Equivalently, to ensure that $\Ex{f\left(\X_{T}\right)}-f(\X^{\star}) \leq \epsilon$, we require
	\begin{align}
		T = \O\Big(n^2\kappa \log\Big(\frac{D_0}{\epsilon}\Big)\Big)
	\end{align}
	for $d \gg 1$. 
\end{proof}
Observe that as compared to gradient descent, the proposed uni-directional subspace descent algorithm requires $\O(n^2)$ as many iterations, since at every iteration, we choose to advance along a single direction out of the possible $\O(n^2)$ directions. A similar result can be established for the RGSD algorithm, as stated in the following corollary. 
\begin{corollary}\label{rgsdic}
	Under Assumptions \ref{a-smooth} and \ref{stronglyconvex}, the uni-directional RGSD algorithm has an iteration complexity of $\O(n^2\kappa \log(\frac{D_0}{\epsilon}))$. 
\end{corollary}
\begin{proof}
	Recall that at each iteration, RGSD selects $(i,j)$ such that $\beta_{ij}^2$ is maximized. Denoting $\beta_{\max}^2 := \max_{(i,j)}\beta_{ij}^2$, it follows that $\xib_{\X_t\X_{t+1}}= -\frac{1}{L} \beta_{\max}\mathbf{G}_{\max}^R(\X_t)$ and 
	\begin{align}
		f\left(\X_{t+1}\right)&\leq 	f\left(\X_t\right)-\frac{\beta_{\max}^2}{L}+ \frac{\beta_{\max}^2}{L^2}\frac{L}{2}
		\;=\;	f\left(\X_t\right)-\frac{\beta_{\max}^2}{2L} \nonumber\\
		&\leq f\left(\X_t\right)-\frac{1}{2dL}\sum_{1\leq j \leq i \leq n}\beta_{ij}^2 \\
		&= f\left(\X_t\right)-\frac{1}{2dL}\|\grd^R f(\X_t)\|_{\X_t}^2 \label{RGSD_uni}
	\end{align}
	which is similar to \eqref{exp_at_t}. The rest of the proof therefore follows in the same way, and hence yields the same iteration complexity. 
\end{proof}

Recall from Table \ref{peritercomplexity} that the per-iteration complexity of RRSD is $\O(n)$ while that of RGSD is $\O(n^2)$ for $f \in \cF$. In light of Theorem \ref{rrsdic} and Corollary \ref{rgsdic}, the overall computational complexity of uni-directional variants of RRSD and RGSD becomes $\O(n^3)$ and $\O(n^4)$, respectively. Hence, the overall complexity of uni-directional RRSD is the same as that of RGD, while the complexity of uni-directional RGSD is higher than that of RGD. Both uni-directional variants however require only $\O(n)$ memory as compared to the $\O(n^2)$ memory required by RGD. 

\subsection{Multi-dimensional update}	
In the multi-directional RRSD, the proof follows in a similar manner if the selected descent direction is an unbiased estimate of the steepest descent direction, at least approximately. Recalling the definition of $\cE_t$, we will require that 
\begin{align}
	\Et{\sum_{(i,j)\in \cE_t} \beta_{ij}^2} \approx \frac{1}{n}\sum_{1\leq j \leq i \leq n}\beta_{ij}^2
\end{align}
where $\Et{\cdot}$ denotes the expectation with respect to the random set $\cE_t$. In practice, to ensure that the selected direction is approximately unbiased, we randomly permute the set $\left\{1,2,\dots,n\right\}$ and then re-shape it into a $2\times \lfloor n/2\rfloor$ matrix $\E$, while ignoring the last element in case $n$ is odd. Subsequently, basis vectors are selected corresponding to each column of $\E$. For the $i$-th column of $\E$, denoted by $[k,l]^\T$, we choose the directions $\mathbf{G}_{kk}^R(\X)$ and $\mathbf{G}_{ll}^R(\X)$ with probability $1/n$ and the direction $\mathbf{G}_{kl}^R(\X)$ with probability $1-1/n$. For this algorithm, it can be seen that 
\begin{align}
	\Et{\sum_{(i,j)\in \cE_t} \beta_{ij}^2} = \frac{1}{c}\sum_{1\leq j \leq i \leq n}\beta_{ij}^2 \label{approxub}
\end{align}
where $c = n$ for $n$ even and $c = n^2/(n-1)$ for $n$ odd. The following theorem provides the iteration complexity result when \eqref{approxub} holds. 

\begin{theorem}
	Under Assumptions \ref{a-smooth} and \ref{stronglyconvex}, and when \eqref{approxub} holds, the multi-directional RRSD algorithm has an iteration complexity of $\O(n\kappa \log(\frac{D_0}{\epsilon}))$.
\end{theorem}

\begin{proof}
	Let the chosen descent direction be denoted by 
	\begin{align}
		\S(\X_t) &= \sum_{(i,j) \in \cE_t} \beta_{ij}\G^R_{ij}(\X_t) = \B_t\Big(\sum_{(i,j)\in\cE_t} \beta_{ij}\Eb_{ij}\Big)\B_t^\T.
	\end{align}
	Substituting  $\xib_{\X_t\X_{t+1}}= -\frac{1}{L} \S(\X_t)$, in \eqref{Lipschitz_cont_grad}, we get
	\begin{align}
		f\left(\X_{t+1}\right)&\leq	f\left(\X_t\right)-\sum_{(i,j)\in\cE_t}\frac{\beta_{ij}^2}{L}+ \sum_{(i,j)\in\cE_t}\frac{\beta_{ij}^2}{L^2}\frac{L}{2}
		\nonumber
		\\
		&=	f\left(\X_t\right)-\sum_{(i,j)\in\cE_t}\frac{\beta_{ij}^2}{2L}. \label{multi_dir_decrse}
	\end{align}
	Taking expectation with respect to the random indices in $\cE_t$, and using \eqref{approxub}, we obtain
	\begin{align}
		\Et{f\left(\X_{t+1}\right)}&\leq  f\left(\X_t\right)-\frac{1}{2cL}\|\grd^R f(\X_t)\|_{\X_t}^2 \label{exp_at_t_subspace} 
	\end{align}
	which is similar to the bound in \eqref{exp_at_t}. Proceeding as earlier, we obtain the desired bound
	\begin{align}
		\E{f\left(\X_{t+1}\right)}-f(\X^{\star})\leq  \left(1-\frac{\mu}{4cL}\right)^{t+1}\big[f(\X_0)-f(\X^{\star})\big] 
	\end{align}
	which translates to an iteration complexity of $\O(n\kappa \log(D_0/\epsilon))$. 
\end{proof}

As compared to the uni-directional case, the iteration complexity of mutli-directional RRSD is $\O(n)$ times better, since we choose to advance along $\O(n)$ directions out of the possible $\O(n^2)$ directions. A similar result again holds for the multi-directional RGSD algorithm, whose proof requires the following preliminary lemma. 
\begin{lemma}\label{greedy_avg}
	Let $\cE$ be the set of pair of indices $\{(i,j)\}$ representing basis directions $\{\G_{ij}^R(\X)\}$ selected by  mutli-directional RGSD algorithm, then the following inequality holds
	\begin{align}
		\sum_{(i,j)\in\cE}\beta_{ij}^2 &\geq   \frac{1}{2n}\|\grd^R f(\X)\|_{\X}^2
	\end{align}
	for all $\X \in \Pn$. 
\end{lemma}
\begin{proof}
	With some abuse of notation, let us denote $\{\beta_{ij}\}_{1 \leq j \leq i \leq n}$ by the linearly indexed terms $\{\beta_m\}_{m\in \cI}$ for $\cI=\{1,2,\cdots,n(n+1)/2\}$. Likewise, let $\cM$ collect the linear indices corresponding to the basis vectors selected by the greedy algorithm, so that $\sum_{(i,j)\in\cE}\beta_{ij}^2 = \sum_{m\in \cM}\beta_{m}^2$. Without loss of generality, let us assume that $\beta_{m_1}^2\geq \beta_{m_2}^2\geq\cdots\geq  \beta_{m_K}^2$ for $m_i \in \cM$, where, $K=\abs{\cM}$.
	
	At the first step, RGSD selects $m_1 = \arg\max_{m\in\cI} \beta_m^2$. Let $\cM(m_1)$ be the set of overlapping basis vectors including $m_1$, so that $m_1 = \arg\max_{m \in \cM(m_1)}\beta_m^2$. For a basis vector $\G_{ij}^R(\X)$ such that $i\neq j$, we have that $|\cM(m_1)|=2n-1$, while for $\G_{ii}^R(\X)$, we have that $|\cM(m_1)|=n$. Therefore, by definition of $\beta_{m_1}$, we have the inequality:
	\begin{align} \label{betam1}
		\beta_{m_1}^2 \geq \frac{\sum_{m\in \cM(m_1)}\beta_m^2}{\abs{\cM_{m_1}}} \geq \frac{\sum_{m\in \cM(m_1)}\beta_m^2}{2n}
	\end{align}
	In the same way, at the $i$-th step, RGSD selects $m_i = \arg\max_{m\in\cM(m_i)} \beta_m^2$ where $\cM(m_i)$ is the set of indices that overlap with $m_i$ (including $m_i$), taken from the set $\cI \setminus \cup_{j=1}^{i-1}\cM(m_j)$. Further, it can be seen that $\abs{\cM(m_i)} \leq 2n$, so that
	\begin{align}\label{betami}
		\beta_{m_i}^2 \geq \frac{\sum_{m\in \cM(m_i)}\beta_m^2}{\abs{\cM_{m_i}}} \geq \frac{\sum_{m\in \cM(m_i)}\beta_m^2}{2n}
	\end{align}
	Hence, summing \eqref{betami} over all $i = 1, \ldots, K$, we have that
	\begin{align}
		\sum_{m\in\cM} \beta_m^2 &\geq \frac{1}{2n}\sum_{m \in \cup_j \cM(m_j)} \beta_m^2 = \frac{1}{2n} \sum_{m\in\cI}\beta_m^2 \label{rgsdstop}\\
		&= \frac{1}{2n}\|\grd^R f(\X)\|_{\X}^2
	\end{align}
	which is the required inequality. The last equality in \eqref{rgsdstop} follows from the stopping criteria of RGSD.  
\end{proof}

\begin{corollary}
	Under Assumptions \ref{a-smooth} and \ref{stronglyconvex}, the multi-directional RGSD algorithm has an iteration complexity of $\O(n\kappa \log(\frac{D_0}{\epsilon}))$.
\end{corollary}
\begin{proof}
	From \eqref{multi_dir_decrse} we have that
	\begin{align}
		f\left(\X_{t+1}\right)&\leq f\left(\X_t\right)-\sum_{(i,j)\in\cE_t}\frac{\beta_{ij}^2}{2L} \leq  f\left(\X_t\right)-\frac{1}{4nL}\|\grd^R f(\X_t)\|_{\X_t}^2 \label{multi_greedy_decrse}
	\end{align}
	where the last inequality in \eqref{multi_greedy_decrse} follows from Lemma \ref{greedy_avg}. Combining \eqref{grad_dominate} and \eqref{multi_greedy_decrse}, we obtain
	\begin{align}
		\Rightarrow f\left(\X_{t+1}\right)-f(\X^{\star})\leq  \left(1-\frac{\mu}{8nL}\right)^{t+1}\big[f(\X_0)-f(\X^{\star})\big] 
		\nonumber
	\end{align}
	which translates to an iteration complexity of $\O\Big(n\kappa \log\Big(\frac{D_0}{\epsilon}\Big)\Big)$.
\end{proof}

In summary, the multi-directional RRSD and RGSD algorithms have an improved iteration complexity of $\O\big(n\kappa\log\big(\frac{D_0}{\epsilon}\big)\big)$ as compared to their uni-directional variants. However, from Table \ref{peritercomplexity}, we can see that the overall computational complexity of multi-direction RRSD and RGSD algorithms for $f \in \cF$ is still $\O(n^3\kappa \log(D_0/\epsilon))$ and $\O(n^3\log(n)\kappa \log(D_0/\epsilon))$, respectively. In other words and as is also the case with the coordinate descent algorithms in the Euclidean case, the overall computational complexity of every subspace descent variants is roughly the same as that of RGD. However, the subspace descent algorithms incur a significantly lower per-iteration computational complexity and hence can be applied to large-scale settings where carrying out even a single iteration of RGD may be prohibitive.

\section{Adaptive step-size selection for function class $\cF$}\label{adapt_step}
In this section, we provide an adaptive step-size selection rule that can simplify the hyperparameter tuning associated with running the proposed algorithm on large-scale problems. While adaptive step-sizes have been used in the context of Riemannian optimization, the present work focuses on achieving the adaptation without sacrificing the computational efficiency of the subspace approach. The need for low-complexity adaptive approaches rules out the use of techniques such as inexact line search (\cite{ferreira2019gradient}) and adaptive gradient methods (\cite{roy2018geometry}, \cite{becigneul2018riemannian}, \cite{kasai2019riemannian}). In this work, we utilize a modified version of the step-size selection strategy proposed for the Euclidean case in \cite{fountoulakis2018flexible}. Specifically, we construct a separate quadratic approximation of the objective along each of the basis directions within the subspace selected as per Sec. \ref{RSD_for_function_class}. Subsequently, the step-sizes along each of these directions is chosen to minimize the corresponding quadratic approximations. The Taylor series expansion of $f$ along a geodesic $\gamma(\lambda)$ with $\gamma'(0) = \V$ can be written as \cite[p. 80]{boumal2023introduction} :
\begin{align}
	f(\gamma\left(\lambda\right))&\approxeq f(\X)+\lambda\left\langle\grd^R \; f(\X),\V\right\rangle_{\X}+\frac{\lambda^2}{2}\left\langle H^Rf(\X)[\V],\V\right\rangle_{\X}\nonumber\\
	&=f(\X)+\lambda\tr{\grd \; f(\X)\V}+\frac{\lambda^2}{2}\left\langle H^Rf(\X)[\V],\V\right\rangle_{\X}\label{second_ordr_approx}.
\end{align}
Recall that $\grd \; f(\X)$ and  $\grd^R \; f(\X)$ denote Euclidean and Riemannian gradients respectively, while $H^Rf(\X)$ denotes the Riemannian Hessian of function $f$. At time instant $t$, the optimal value of $\lambda_{t,\V}$ which minimizes the second order approximation \eqref{second_ordr_approx} for  geodesically convex function $f$ along the  geodesic with initial direction $\V$ is given by
\begin{align}
	\lambda_{t,\V}^{\star} &= -\frac{\tr{\grd \; f(\X)\V}}{\left\langle H^Rf(\X)[\V],\V\right\rangle_{\X}}.
\end{align}
For each $\G_{ij}^R$, the corresponding $\lambda_{t,\G_{ij}^R}^{\star}$ takes on the role of $-\alpha_t\beta_{ij}$ in \eqref{multi_SD_update}. It is noteworthy that the computation of $\beta_{ij}$ is inherently part of the calculation of $\lambda_{t,\G_{ij}^R}^{\star}$. Therefore, in the $t$-th iteration, the descent direction is given by
	\begin{align}
		\S(\X_t) &= \sum_{(i,j) \in \cE_t} \lambda_{t,\G_{ij}^R}^{\star}\G^R_{ij}(\X_t) = \B_t\Bigg(\sum_{(i,j)\in\cE_t} \lambda_{t,\G_{ij}^R}^{\star}\Eb_{ij}\Bigg)\B_t^\T 
	\end{align}
Since the basis vectors in each subspace are non-overlapping, it follows from \eqref{exp_sd} that the corresponding optimal step-sizes along each of the basis directions can be selected independently. Hence, the update equation becomes
\begin{align}
	\X_{t+1} &= \Expxt(\S(\X_t))= \B_t\exp\Bigg(\sum_{(i,j)\in \cE_t}\lambda_{t,\G_{ij}^R}^{\star}\Eb_{ij}\Bigg)\B_t^\T
\end{align}
Notably, computing the optimal $\lambda_{t,\G_{ij}^R}^{\star}$ for the function class $\cF$ has a computational complexity of $O(n)$ for each $\G_{ij}^R$. Detailed calculations are provided in Appendix \ref{apndx_step_size_fun_cls}. In case we expect $f$ to grow more rapidly than its quadratic approximation, it may be necessary to scale down $\lambda_{t,\G_{ij}^R}^{\star}$ by dividing it by a scaling factor to achieve a reduction in the function value. The scaling factor must be tuned for a given problem.

It is remarked that for Riemannian gradient descent, the step-size selection based on quadratic approximation is computationally demanding; see also Appendix \ref{apndx_step_size_fun_cls}. Therefore, the proposed step-size selection approach is not practical for Riemannian gradient descent.

\section{Experimental results}
In this section, we test the performance of the proposed subspace descent variants over strongly g-convex functions in $\cF$. The performance of the proposed algorithms is compared with that of the RGD algorithm. Of particular interest are the large scale settings, where accelerated and second-order methods have prohibitively high complexity and are impractical. We also confirm numerically that the per-iteration complexity of RGD grows as $\O(n^3)$ while that of multi-directional RRSD and RGSD grows as $\O(n^2)$, and that of uni-directional RRSD grows as $\O(n)$. All simulations are performed in MATLAB on a system with 512 GB RAM. To ensure sufficient memory for large-scale settings, we do not pre-allocate the intermediate variables and clear them after every use. Care is taken to ensure that the key steps of different algorithms are implemented similarly, so that their run times reflect their computational complexity and can be compared directly. 

Comparison with the coordinate descent approach proposed in \cite{gutman2022coordinate} is not included, as no specific low-complexity algorithm for handling the function minimization over SPD manifold (cf. \eqref{problem}) was provided. We note however that the general approach provided in \cite{gutman2022coordinate} can be seen as including the uni-directional RRSD algorithm as a special case, but not the other variants.

Consider the function $f \in \cF$ given by
\begin{align}
	f\left(\X\right)&=\tr{\mathbf{CX}^{-1}+\mathbf{DX}}+k \log\det\left(\X\right)\nonumber\\
	&= \tr{\B^{-1}\C\B^{-\T}}+\tr{\B^{\T}\D \B}+2k \log\det\left(\B\right) \label{fsimulations}
\end{align}
where $\B = \cL(\X)$, $\C, \D \succ 0$, and $k\in \Rn$. Further, $f$ is  $\mu$-strongly g-convex with $\mu = \min\{\lambda_{\min}(\C),\lambda_{\min}(\D)\}$ (see Lemma \ref{lemma_sc}) and has the gradient
\begin{align}
	\grd f\left(\X\right)&= \D -\X^{-1}\C\X^{-1}+k\X^{-1}\\
	\grd^R f\left(\X\right)&= \X\D \X-\C+k\X
\end{align}
In general, gradient field of $f$ may not be Lipschitz smooth, unless we restrict $\X$ to a norm ball of radius $R$. For example, gradient vector field of function  $\tr{\X^{-1}+\X}$ is not Lipschitz smooth but when $\X$ is confined to a norm ball of radius $R$, its gradient vector field is $(e^{R}+e^{-R})$-Lipschitz smooth (see Lemma \ref{lemma_l_smooth}). For a descent algorithm, assuming $\mu$-strong g-convexity, it can be seen that all iterates would lie in norm ball of radius $\sqrt{2(f(\X_0)-f(\X^{\star}))/\mu}$.

The function form in \eqref{fsimulations} is motivated from its use in maximum a-priori (MAP) estimation of the covariance matrix of an $n$-variate Gaussian random variable with Wishart prior. Specifically, the likelihood function for independent and identically distributed data points $\x_1, \ldots, \x_N \sim \cN(\mub, \Sig)$ is given by 
\begin{align}
	L(\mub, \Sig, \{\x_i\}_{i=1}^N) = c_1 \det\left(\Sig\right)^{-n/2}\exp\left(-\frac{1}{2}\tr{\C_d\Sig^{-1}}\right)
\end{align}
where, 
\begin{align}
	\mathbf{C}_d&= \sum_{i=1}^{N}\left(\x_i-\boldsymbol{\mu}\right)\left(\x_i-\boldsymbol{\mu}\right)^\T
\end{align}
and $\Sig$ is Wishart with parameters $(n,p,\S)$ so that \cite[p.87]{gupta1999matrix}
\begin{align}
	p\left(\Sig\big|n,p,\S\right) &= c_2\det\left(\Sig\right)^{\frac{1}{2}\left(n-p-1\right)}\exp\left(-\frac{1}{2}\tr{\S^{-1}\Sig}\right)
\end{align}
for $\S\succ\mathbf{0}$ and $n\geq p$. Hence, the MAP estimator of $\Sig$ can be obtained by solving 
\begin{align}
	\hat{\Sig}_{\text{MAP}} &= \min_{\Sig}\left\{\tr{\mathbf{C}_d\Sig^{-1}}+\tr{\mathbf{S}^{-1}\Sig}+c_3\log\det\left(\Sig\right)\right\} \label{Gauss_map}
\end{align}
We remark that it is common to instead associate an inverse Wishart prior when dealing with MAP estimation of the covariance matrix. Since inverse Wishart is the conjugate prior of the covariance matrix of a Gaussian distribution, it yields closed form MAP estimates. The formulation in \eqref{Gauss_map} however allows for a Wishart prior, and can similarly be modified to allow for other priors, including inverse Wishart and normal-Wishart priors. 

We also note that for some specific choices of $\C$,  $\D$, and $k$, the solution can be found in closed-form: 
\begin{itemize}
	\item  for $\D =\I$ (Identity matrix) and $k=0$, we have that $\grd^R f\left(\X^{\star}\right)= \X^{\star}\X^{\star}-\C=0$, which yields $\X^{\star} = \C^{1/2}$, and
	\item For $\C=\mathbf{0}$ and $k=-1$, we have that $\grd^R f\left(\X^{\star}\right)= \X^{\star}\D \X^{\star}-\X^{\star}=0$, which yields $\X^{\star}=\D ^{-1}$.
\end{itemize}
For $f$ in \eqref{fsimulations}, the RGD updates take the following form:
\begin{align}
	\X_{t+1}
	&= \B_t\exp\left(-\alpha_t\B_t^{-1}\grd^R f\left(\X_t\right)\B_t^{-\T}\right)\B_t^\T\nonumber\\
	&= \B_t\exp\left(\alpha_t\left(\B_t^{-1}\C\B_t^{-\T}-\B_t^{\T}\D \B_t -k\cdot\I \right)\right)\B_t^\T.
\end{align}
For RRSD and RGSD algorithms, the coefficient $\beta_{ij}$ is given by 
\begin{align}
	\beta_{ij}&=  \tr{\left[\B_t^{\T}\D \B_t-\B_t^{-1}\C\B_t^{-\T}+k\cdot\I\right]\mathbf{E}_{ij}}.
\end{align}
As discussed in earlier section, $f(\X_t)$ and $\beta_{ij}$ can be written using the following
intermediate variables:
\begin{align}
	\M_{1}(\X_t)&=\B_t^{\T}\D \B_t\\
	\M_{2}(\X_t)&=\B_t^{-1}\C\B_t^{-\T}
\end{align} 
so that
\begin{align}
	f\left(\X_t\right)&=\tr{\M_{1}(\X_t)}+\tr{\M_{2}(\X_t)}+k\cdot \log\det\left(\X_t\right)\nonumber\\
	\beta_{ij}&=  \sqrt{2}^{\I\left(i\neq j\right)} \left[\M_{1}(\X_t)-\M_{2}(\X_t)+k\cdot\I\right]_{ij}\label{beta_in_subclass}
\end{align}
The intermediate matrices $\M_{1}(\X_t)$ and $\M_{2}(\X_t)$ can be updated in recursive manner as specified in \eqref{m1_update}-\eqref{m2_update}.

We consider two choices of $k$, namely $-1$ and $0$, which correspond to the functions:
\begin{align}
	f_1\left(\X\right)&=\tr{\C\X^{-1}+\D\X}-\log\det\left(\X\right)\\
	f_2(\X)&=\tr{\C\X^{-1}+\D\X}
\end{align}
respectively. We remark that $\log\det(\X)$ is geodesically linear in the SPD manifold, and hence both $f_1$ and $f_2$ have the same condition number. However, the performance of gradient-based methods often depends on the local condition number in the vicinity of $\X^\star$, which could be different for $f_1$ and $f_2$ (see Appendix \ref{apndx_codn_number}). For instance, if we take 
\begin{align}
	\C=\begin{bmatrix}
		5.6667 & 10.0000&5.8889\\
		10.0000&26.2222& 17.5556\\
		5.8889& 17.5556 & 12.1111
	\end{bmatrix}
\end{align}
and $\D=\I$, then, it can be numerically verified that the condition numbers of $f_1$ and $f_2$ are $12.87$ and $104.88$ at their respective optima. Indeed, as we shall see in the simulations as well, $f_1$ is relatively well-behaved and various algorithms converged faster on $f_1$ than on $f_2$. 

\subsection{Per-iteration complexity}
We begin with empirically verifying the per-iteration time complexity of the proposed subspace descent algorithms. As we are only concerned about the per-iteration complexity and its evolution with $n$, we take $\C$ and $\D$ to be symmetric matrices that may not necessarily be positive definite in order to save on the simulation time. Fig. \ref{time_per_iter}(a) shows the time taken to run each iteration of RGD and RRSD algorithms. It can be seen from the figure that, as expected, the per iteration time complexity of multi-directional RRSD grows as $\O(n^2)$ while that of RGD grows as $\O(n^3)$. 
\begin{figure}[H]
	\includegraphics[width=\textwidth]{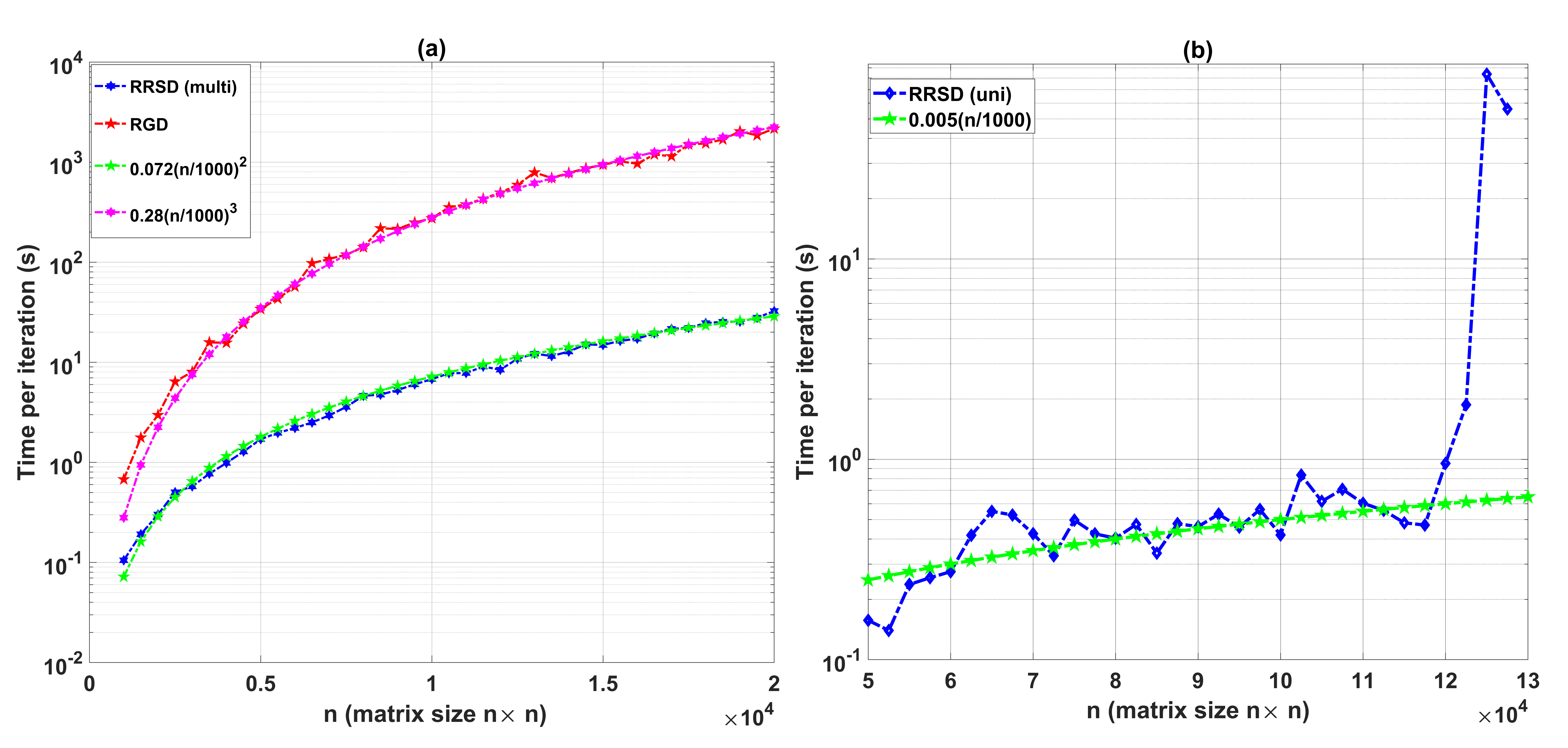}
	\caption{Comparison of  time per iteration of RGD, uni-directional RRSD  (RRSD (uni)) and multi-directional RRSD (RRSD (multi)) methods}
	\label{time_per_iter}
\end{figure}
We also consider larger-scale settings where even $\O(n^2)$ complexity is impractical, and it becomes necessary to implement the uni-directional RRSD whose per-iteration complexity grows only linearly with $n$. Fig. \ref{time_per_iter}(b) shows the time per-iteration of the RRSD algorithm. The performance of RGD is not shown, as we observed that it ran out-of-memory for $n=60000$. It can be seen that the per-iteration time complexity of RRSD grows almost linearly for $n \leq 122500$.  For larger $n$ however, the time-complexity increase abruptly due to the use of swap memory, and the system runs out of memory for  $n=130000$. We comment that it may be possible to go beyond this limit by reading/writing only parts of different variables directly from the disk, as explained in \ref{rrsdfuni}. 

\subsection{Performance comparison}
Next, we compare the empirical performance of various proposed algorithms and that of RGD.  To assess performance of the proposed algorithm, we simulated SPD matrix $\C\in \mathbb{R}^{n\times n}$ for $n\in \{100,500,1000\}$ as $\C=\C_{temp}\C_{temp}^{\T}/n^2$, 
	where,  $\C_{temp}$ is $n\times n$ matrix of pseudorandom integers drawn from the discrete uniform distribution on the interval [1,10]. and $\D = \I$. All algorithms are initialized with $\X_0 = \I$. The step-sizes for RRSD and RGSD are calculated by adaptive step-size selection method described in Sec. \ref{adapt_step}  while the step-size for RGD is set to $0.1$, which is the largest value for which RGD converges. For the problems at hand, it is possible to calculate the optimal values in closed-form. Specifically, if the eigenvalue decomposition of $\C$ is given by $\C=\U\Sigma_{\C}\U^{\T}$, then the minimizer of $f_1$ is 
$\X_1^{\star}=\U\Sigma_{\X_1^{\star}}\U^{\T}$, where $[\Sigma_{\X_1^{\star}}]_{ii}=\frac{1+\sqrt{1+4[\Sigma_{\C}]_{ii}}}{2}$ and minimizer of $f_2$ is $\C^{1/2}$.

Recall that the theoretical performance of various algorithms was characterized in terms of the computational complexity. In practice, there are different ways to benchmark the performance. Here, we use the following metrics for performance comparison: 
\begin{enumerate}[1)]
	\item Number of basis directions, $\sum_t K_t$,
	\item Number of entries of $\F(\X)$ that must be calculated over all iterations,
	\item Total number of floating point operations. 
\end{enumerate}
Optimality gaps for different values of $n$ and different metrics are shown in Figs. \ref{fig_logdet} and \ref{fig_withoutlogdet} for $f_1$ and $f_2$, respectively.

\begin{figure}[H]
	\centering
	\includegraphics[width=\textwidth]{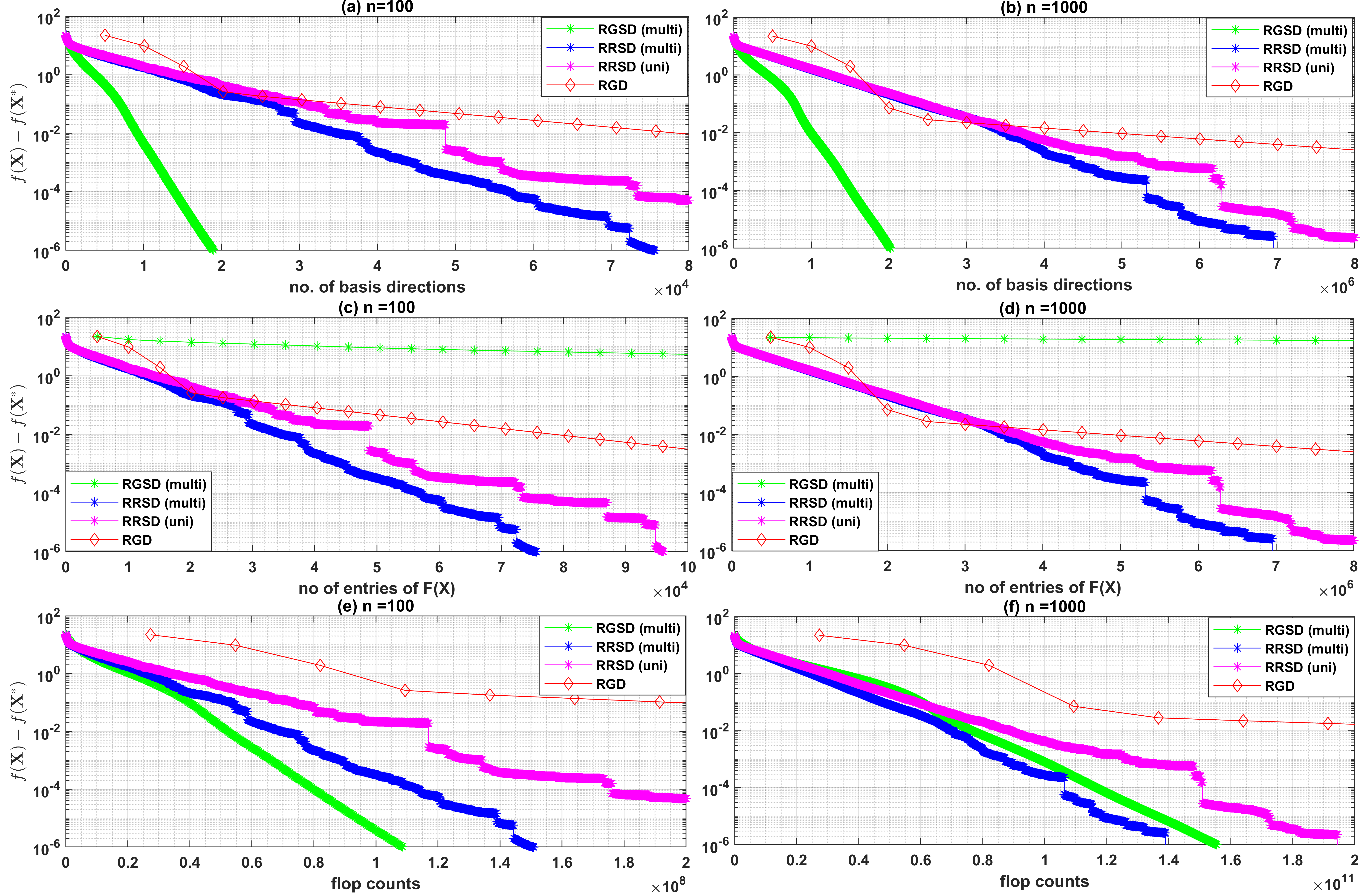}
	\caption{Comparison of  RGD and subspace-descent methods for function $f_1(\X)$; RGSD (multi) :  multi-directional RGSD, RRSD (multi) : multi-directional RRSD, RRSD (uni) : uni-directional RRSD.}
	\label{fig_logdet}
\end{figure}
\begin{figure}[H]
	\centering
	\includegraphics[width=\textwidth]{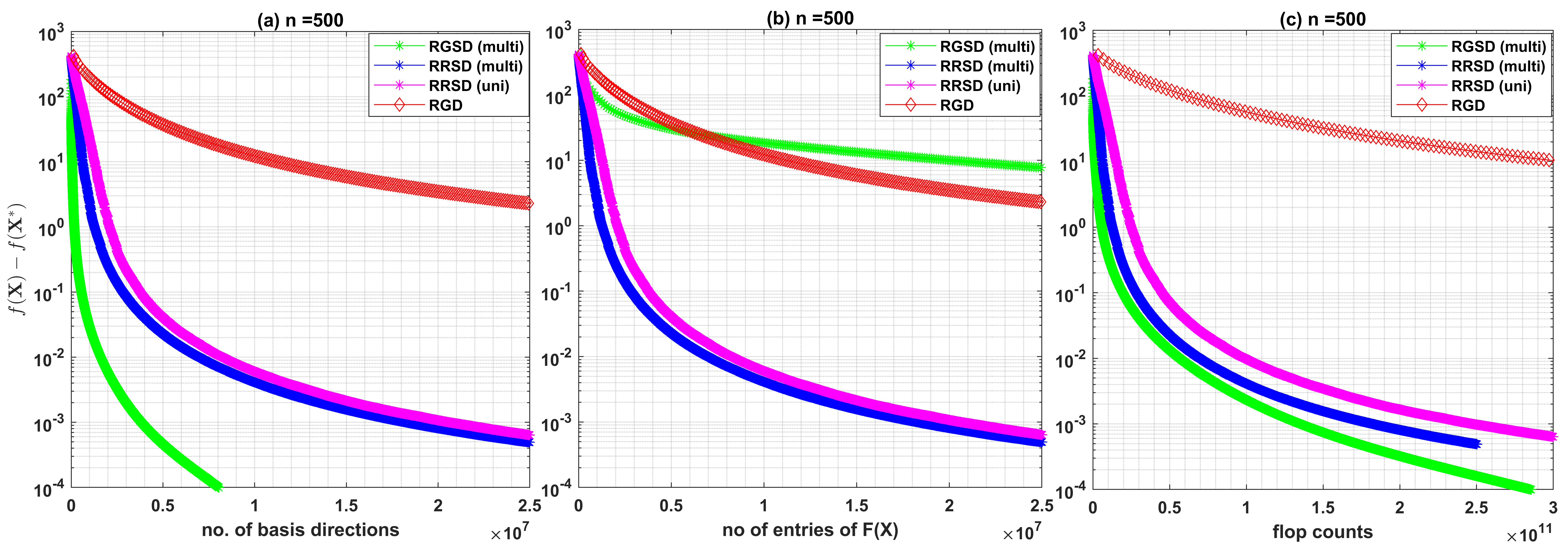}
	\caption{Comparison of  gradient descent and subspace descent methods for function $f_2(\X)$}
	\label{fig_withoutlogdet}
\end{figure}

Recall that RRSD  and RGSD algorithms use at most $n$ basis directions per-iteration while RGD uses all $\frac{n(n+1)}{2}$ directions at every iteration. From Figs. \ref{fig_logdet}(a), \ref{fig_logdet}(b) and \ref{fig_withoutlogdet}(a), we can see that both RRSD and RGSD algorithms outperform RGD. Interestingly, for this metric, the performance of mutli-directional RGSD algorithm is significantly better than that of all other algorithms, since it seeks to choose the best possible directions at every iteration.

Recall that both the RRSD algorithm variants need at most $n$ entries of $\F(\X)$ per iteration. In contrast, although RGSD algorithm also utilizes at most $n$ entries of $\F(\X)$ at every iteration, it still requires all the entries of $\F(\X)$ in order to select the greedy directions. The observation is confirmed from Figs. \ref{fig_logdet}(c), \ref{fig_logdet}(d) and \ref{fig_withoutlogdet}(b), which show RGSD as the worst performing among all the algorithms. 

Finally, we compare the different algorithms on the basis of their flop count, which should reflect the actual run-time and be largely independent of the system and implementation. Figs. \ref{fig_logdet}(e), \ref{fig_logdet}(f) and \ref{fig_withoutlogdet}(c) show that multi-directional RRSD and RGSD are both superior to uni-directional RRSD, which is again much better than the RGD algorithm. Of these, RGSD is better for small $n$ only, as the cost of calculating and sorting the entries of $\F(\X)$ starts to dominate for large $n$.

\section{Conclusion and future work}
Minimization problems over the symmetric positive definite (SPD) manifold arise in a number of areas, such as kernel matrix learning, covariance estimation of Gaussian distributions, maximum likelihood parameter estimation of elliptically contoured distributions and parameter estimation in Gaussian mixture model problems. Recent years have seen the development of Riemannian optimization algorithms that seek to be more efficient at solving these problems by exploiting the structure of the SPD manifold. However, the Riemannian gradient descent and other related algorithms generally require costly matrix operations like matrix exponentiation and dense matrix multiplication  at every iteration, and hence incur a complexity of at least $\O(n^3)$ at every iteration. Motivated by coordinate descent algorithms popular in the Euclidean case, we put forth Riemannian subspace descent algorithms that are able to achieve lower per-iteration complexities of $\O(n)$ and $\O(n^2)$. The proposed algorithms achieve this reduction by identifying special subspaces which allow low-complexity update of the Cholesky factors of the iterates and by restricting to a class of functions over which the Riemannian gradients can be efficiently calculated. The performance of the proposed algorithms is evaluated for large-scale covariance estimation problems, where they are shown to be faster than the Riemannian gradient descent algorithm.

The proposed algorithms can be viewed as Riemannian counterparts of the classical block coordinate descent (BCD) algorithm. Naturally, it remains to see if the advances in BCD carry over to the Riemannian setting. In the Euclidean case for instance, a large number of variable selection or block selection strategies have been explored, and similar strategies can be explored in the Riemannian case also. Likewise, iteration complexity analysis of the proposed algorithms for other interesting cases, such as when the function is non-convex and/or satisfies the Polyak-Łojasiewicz (PL) inequality, would be of great importance. Finally, it remains to see if faster versions of the proposed algorithms, possibly using momentum or variance-reduction techniques, can be developed for the finite sum problem, briefly discussed in Sec. \ref{RSD_for_function_class}. 


\appendix
\section{}

\begin{lemma}\label{lemma_sc}
The function $f(\X)=\tr{\C\X^{-1}+\D\X}$ is $\mu$-strongly g-convex for $\C \succ 0$ and $\D \succ 0$ with $
	\mu=\min\left(\lambda_{\min}(\C),\lambda_{\min}(\D)\right)$. 
\end{lemma}
\begin{proof}
	We begin with noting that
	\begin{align}
		df(\X)&=-\tr{\X^{-1}\C\X^{-1}d\X}+\tr{\D d\X}\\
		\Rightarrow \grd f(\X)&= \D-\X^{-1}\C\X^{-1}.
	\end{align}
The Euclidean hessian can be evaluated as follows 
	\begin{align}
		\G(\X)&=\grd f(\X)=\D-\X^{-1}\C\X^{-1}\\
		d\G(\X)&=\X^{-1}d\X\X^{-1}\C\X^{-1}+\X^{-1}\C\X^{-1}d\X\X^{-1}\\
\Rightarrow		Hf(\X)[\V]&= \X^{-1}\V\X^{-1}\C\X^{-1}+\X^{-1}\C\X^{-1}\V\X^{-1}
	\end{align}
For the Riemannian manifold $\Pn$, Euclidean and Riemannian Hessians are related as \cite{ferreira2019gradient}: 
\begin{align}
H^Rf(\X)[\V]&=\X	Hf(\X)[\V]\X+\frac{1}{2}\big[\V\grd f(\X)\X+\X\grd f(\X)\V\big]\\
\Rightarrow		\left\langle H^Rf(\X)[\V],\V\right\rangle_{\X}&=\tr{\D\V\X^{-1}\V}+\tr{\V\X^{-1}\V\X^{-1}\C\X^{-1}}.\label{hess_eqn_fun}
	\end{align}
Suppose that $f(\X)$ is $\mu$-strongly g-convex. Then we must have that
	\begin{align}
		&\left\langle H f(\X)[\V],\V\right\rangle_{\X}\geq \mu\tr{\V\X^{-1}\V\X^{-1}}\\
	\Leftrightarrow 	&\tr{\D\V\X^{-1}\V}+\tr{\V\X^{-1}\V\X^{-1}\C\X^{-1}}-\mu\tr{\V\X^{-1}\V\X^{-1}}\geq0\\
	\Leftrightarrow		&\tr{\V\X^{-1}\V\left(\D+\X^{-1}\C\X^{-1}-\mu\X^{-1}\right)}\geq0 \label{sc_condition_original}
	\end{align}
which holds if  and only if
	\begin{align}
		\D+\X^{-1}\C\X^{-1}-\mu\X^{-1}\succcurlyeq0
	\end{align}
or equivalently
\begin{align}
	\lambda_i(\D+\X^{-1}\C\X^{-1}-\mu\X^{-1})&\geq 0\; & 1 \leq &i \leq n. \label{strong_convx_cond}
\end{align}
To derive condition on $\mu$ such that \eqref{strong_convx_cond} is satisfied let us divide it into two cases (a) $\lambda_i(\X)\leq 1$ and (b) $\lambda_i(\X) > 1$. 

\noindent \textbf{Case $\lambda_i(\X)\leq 1$: }For $\D \succ 0$ we have that \cite[p.117]{george2008matrix}
\begin{align}
\D+\X^{-1}\C\X^{-1}-\mu\X^{-1}&\succeq  \X^{-1}\C\X^{-1}-\mu\X^{-1}=\X^{-\frac{1}{2}}\left(\X^{-\frac{1}{2}}\C\X^{-\frac{1}{2}}-\mu\I\right)\X^{-\frac{1}{2}}. \label{case_eig_Xleq1}
\end{align}
For two symmetric matrices $\A$ and $\B$, it is known that $\A\succeq\B\Leftrightarrow$ $\R^{\T}\A\R\succeq \R^{\T}\B\R$ for non-singular $\R$ \cite[p.227]{george2008matrix}. Therefore, \eqref{case_eig_Xleq1} is equivalent to
\begin{align}
	\X^{\frac{1}{2}}\left(\D+\X^{-1}\C\X^{-1}-\mu\X^{-1}\right)\X^{-\frac{1}{2}}&\succeq \X^{-\frac{1}{2}}\C\X^{-\frac{1}{2}}-\mu\I
\end{align}
which implies that  \cite[p.228]{george2008matrix}
\begin{align}
	\lambda_i\left(\X^{\frac{1}{2}}\left(\D+\X^{-1}\C\X^{-1}-\mu\X^{-1}\right)\X^{-\frac{1}{2}}\right)\geq \lambda_i\left(\X^{-\frac{1}{2}}\C\X^{-\frac{1}{2}}-\mu\I\right). \label{lambda_i_case_1}
\end{align}
If $ \mu$ is selected such that $\lambda_i\left(\X^{-\frac{1}{2}}\C\X^{-\frac{1}{2}}-\mu\I\right)\geq 0$	then the relation \eqref{lambda_i_case_1} ensures that \eqref{strong_convx_cond} is satisfied for that index $i$ for which $\lambda_i(\X)\leq 1$. This will now allow us to derive the condition on $\mu$. Note  that
		\begin{align}
			\lambda_i\left[\X^{-\frac{1}{2}}\C\X^{-\frac{1}{2}}-\mu\I\right]= \lambda_i\left[\X^{-\frac{1}{2}}\C\X^{-\frac{1}{2}}\right]-\mu=\lambda_i\left(\X^{-1}\C\right)-\mu
		\end{align}
	where the last equality follows from the fact that for two symmetric positive definite matrices $\A$ and $\B$, the eigenvalues of $\A\B$ and $\B\A$ are the same. We further have that $\lambda_i(\A\B)\geq \lambda_i(\A)\lambda_{\min}(\B)$ \cite[p.119]{george2008matrix}, so that
	\begin{align}
			\lambda_i\left(\X^{-1}\C\right)-\mu&\geq \lambda_i\left(\X^{-1}\right)\lambda_{\min}(\C)-\mu.
		\end{align}
	Hence, since $\lambda_i(\X)\leq 1$, we have that	
\begin{align}
	\lambda_i\left(\X^{-1}\right)\lambda_{\min}(\C)-\mu\geq 0\text{ if } \lambda_{\min}(\C)\geq \mu
\end{align}.

\noindent \textbf{Case $\lambda_i(\X)>1$: } Similar to previous case, for $\C\succ 0$ we have that $\X^{-1}\C\X^{-1}\succ 0$, which implies that
		\begin{align}
			\D+\X^{-1}\C\X^{-1}-\mu\X^{-1}&\succeq  \D-\mu\X^{-1}
		\end{align}
If $\mu$ is chosen such that inequality
\begin{align}
	\lambda_i\left(\D-\mu\X^{-1}\right)\geq 0 \Leftrightarrow \lambda_i(\D)\geq \mu\lambda_i(\X^{-1}) \label{eig_X_geq1}
\end{align}
 is satisfied, it ensures that \eqref{strong_convx_cond} is also satisfied for the index $i$ where $\lambda_i(\X)>1$. 
The inequality \eqref{eig_X_geq1} is satisfied if $\lambda_{\min}(\D)\geq \mu$.

Combining the two cases, it follows that $\tr{\C\X^{-1}+\D\X}$ is g-strongly $\mu$-convex with 
	\begin{align}
		\mu=\min\left(\lambda_{\min}(\C),\lambda_{\min}(\D)\right)
	\end{align}
\end{proof}
\begin{lemma}\label{lemma_l_smooth}
The gradient vector field of the funcion $f(\X)=\tr{ \X^{-1}+ \X} $ is not Lipschitz smooth in general, but is Lipschitz smooth when $\X$ lies in a compact set. 
	\end{lemma}
\begin{proof}
For   $f(\X)=\tr{\X^{-1}+\X}$, we have the following relation from \eqref{hess_eqn_fun}
	\begin{align}
		\left\langle H^Rf(\X)[\V],\V\right\rangle_{\X}&=\tr{\V\X^{-1}\V}+\tr{\V\X^{-1}\V\X^{-2}}
	\end{align}
For $L$-smoothness of the gradient vector field, following inequality must be satisfied 
\begin{align}
	\left\langle H\;f(\X)[\V],\V\right\rangle_{\X}=\tr{\V\X^{-1}\V}+\tr{\V\X^{-1}\V\X^{-2}}&\leq L\|\V\|_{\X}^2\\
	\Leftrightarrow \tr{\V\X^{-1}\V\left(\I+\X^{-2}-L\X^{-1}\right)}&\leq 0 \label{ls_condition_original}
\end{align}
which is true if and only if
\begin{align}
	\I+\X^{-2}-L\X^{-1}&\preceq 0\\
	\Leftrightarrow\lambda_i^2(\X)+1-L\lambda_i(\X)&\leq 0; \;\;    1\leq i\leq n \label{ls_condition_derived_2} \\
	\Leftrightarrow \frac{L-\sqrt{L^2-4}}{2}\leq \lambda_i(\X)&\leq  \frac{L+\sqrt{L^2-4}}{2} ; \;\;    1\leq i\leq n  \label{ls_condition_derived_3}
\end{align}
which may not necessarily hold in general. The condition \eqref{ls_condition_derived_2} is not true for $\X$ that does not satisfy \eqref{ls_condition_derived_3}. Therefore the gradient vector field of the  function $f(\X)=\tr{\X^{-1}+\X}$ is not Lipschitz smooth.

Next, let us consider the case when $\X$ lies in a compact set. Without loss of generality, we can assume that the set is contained within a norm ball of radius $R$ around $\X=\I$ given by 
\begin{align}
	\cB=\left\{\X:d(\X,\I)\leq R\right\}
\end{align}
where,
\begin{align}
	d(\X,\I)^2&=\|\log(\I^{-1/2}\X\I^{-1/2})\|_F^2\\
	&=\|\log(\X)\|_F^2 =\sum_{i=1}^{n}(\log(\lambda_i(\X)))^2
\end{align}
which implies that
\begin{align}
	-R&\leq \log(\lambda_i(\X))\leq R\Leftrightarrow e^{-R}\leq \lambda_i(\X)\leq  e^{R}. \label{norm_ball_constraint}
\end{align}
We note that $L=e^{-R}+e^{R}$ satisfies \eqref{ls_condition_derived_3} and \eqref{norm_ball_constraint}. Therefore, over a norm ball of radius $R$ around the optimum point $\X=\I$, the function $f(\X)=\tr{\X^{-1}+\X}$ has $(e^{-R}+e^{R})$-Lipschitz smooth gradient vector field.
\end{proof}
\subsection{Effect of geodesically linear function on the condition-number around the optimum point}\label{apndx_codn_number}
First consider the function
\begin{align}
	f_2(\X)&=\tr{\C\X^{-1}+\X}.
\end{align}
For $\mu$-strongly g-convex function we have that
\begin{align}
	&\left\langle Hf_2(\X)[\V],\V\right\rangle_{\X}=\tr{\V\X^{-1}\V}+\tr{\V\X^{-1}\V\X^{-1}\C\X^{-1}}\geq  \mu\|\V\|_{\X}^2\\
	\Leftrightarrow &\tr{(\V\X^{-1}\V(\I+\X^{-1}\C\X^{-1}-\mu\X^{-1}))}\geq 0\\
	\Leftrightarrow  &\I+\X^{-1}\C\X^{-1}-\mu\X^{-1}\succcurlyeq  0. \label{strong_conx_f_2}
\end{align}
At the optimum point $\X^{\star}=\C^{1/2}$ the strong convexity condition \eqref{strong_conx_f_2} simplifies to
\begin{align}
	\I+\I-\mu\C^{-1/2}&\succcurlyeq 0\Rightarrow  \mu =2\sqrt{\lambda_{\min}(\C)}.
\end{align}
Similarly for  $L$-smoothness of $f_2$, following condition must be satisfied
\begin{align}
	&\left\langle H^Rf_2(\X)[\V],\V\right\rangle_{\X}=\tr{\V\X^{-1}\V}+\tr{\V\X^{-1}\V\X^{-1}\C\X^{-1}}\leq  L\|\V\|_{\X}^2\\
	\Leftrightarrow &\tr{(\V\X^{-1}\V(\I+\X^{-1}\C\X^{-1}-L\X^{-1}))}\leq 0\\
	\Leftrightarrow  &\I+\X^{-1}\C\X^{-1}-L\X^{-1}\preccurlyeq  0. \label{lips_smooth_f_2}
\end{align}
At the optimum point $\X^{\star}=\C^{1/2}$ the $L$-Lipschitz smooth condition \eqref{lips_smooth_f_2} simplifies to
\begin{align}
	\I+\I-L\C^{-1/2}&\preccurlyeq 0 \Rightarrow  L =2\sqrt{\lambda_{\max}(\C)}.
\end{align}
Therefore, the condition number of the function $f_2$ at the optimum point $\X^{\star}=\C^{1/2}$ is 
\begin{align}
	\kappa_2=\frac{2\sqrt{\lambda_{\max}(\C)}}{2\sqrt{\lambda_{\min}(\C)}}=\sqrt{\frac{\lambda_{\max}(\C)}{\lambda_{\min}(\C)}}.
\end{align}
Now let us consider the function 
\begin{align}
	f_1(\X)&=\tr{\C\X^{-1}+\X}-\log\det(\X)\\
	\grd^Rf_1(\X)&=\X^2-\X-\C.
\end{align}
The minimizer of $f_1$ satisfies 
\begin{align}
	\X_{\star}^2-\X_{\star}-\C&=0\Leftrightarrow\I-\X_{\star}^{-1}=\X_{\star}^{-1}\C\X_{\star}^{-1}.\label{opt_f_1}
\end{align}
If we decompose $\C=\U\Sigma_{\C}\U^{\T}$, then 
\begin{align}
	\X_{\star}&=\U\Sigma_{\X_{\star}}\U^{\T}\\
	[\Sigma_{\X_{\star}}]_{ii}&=\frac{1+\sqrt{1+4[\Sigma_{\C}]_{ii}}}{2}.
\end{align}
For $f_1$ to be $\mu$-strongly g-convex, we require that
\begin{align}
	&\I+\X^{-1}\C\X^{-1}-\mu\X^{-1}\succcurlyeq  0 \label{strong_conx_f_1}
\end{align}
We observe that at the optimum point $\X_{\star}$ satisfying \eqref{opt_f_1}, the strong convexity condition \eqref{strong_conx_f_1} simplifies to
\begin{align}
	\I+\I-\X_{\star}^{-1}-\mu\X_{\star}^{-1}&\succcurlyeq 0 \Leftrightarrow  \X_{\star}  \succcurlyeq\frac{\mu+1}{2}\I\\
	\mu&=\sqrt{1+4\lambda_{\min}(\C)}.
\end{align}
In the same way, the $L$-smoothness condition is given by
\begin{align}
	&\I+\X^{-1}\C\X^{-1}-L\X^{-1}\preccurlyeq  0 \label{lips_smooth_f_1}
\end{align}
At the optimum point, \eqref{lips_smooth_f_1} simplifies to
\begin{align}
	\I+\I-\X_{\star}^{-1}-L\X_{\star}^{-1}&\preccurlyeq  0 \Leftrightarrow \frac{L+1}{2}\I  \succcurlyeq \X_{\star}\\
	L&=\sqrt{1+4\lambda_{\max}(\C)}.
\end{align}
Therefore, the condition number at the optimum point $\X_{\star}$ is 
\begin{align}
	\kappa_1=\frac{\sqrt{1+4\lambda_{\max}(\C)}}{\sqrt{1+4\lambda_{\min}(\C)}}=\sqrt{\frac{1+4\lambda_{\max}(\C)}{1+4\lambda_{\min}(\C)}}.
\end{align}
If the matrix $\C$ has the minimum eigenvalue $\lambda_{\min}(\C) \ll1$, then adding $-\log\det(\X)$ to the function $f_2(\X)$ improves the condition number from $\kappa_2=\sqrt{\frac{\lambda_{\max}(\C)}{\lambda_{\min}(\C)}}$ to $\kappa_1\approx\sqrt{1+4\lambda_{\max}(\C)}$ at $\X^\star$, though the overall condition number remains the same. 

\section{Step-size selection for function class} \label{apndx_step_size_fun_cls}
The function $D\left(f\right)_{\X}(\V)$ denotes the directional derivative of function $f(\X)$ in the direction of $\V$ at point $\X$ in the Euclidean geometry.

The directional derivatives of component functions $g_{1,p}, g_{2,q}, g_3, g_{4,r}, g_{5,s}$ and $g_{6,m}$ are provided as follows:
\begin{align}
	D\left(g_{1,p}\right)_{\X}(\V)&= -\tr{\C_p\X^{-1}\V\X^{-1}}=-\tr{\X^{-1}\C_p\X^{-1}\V}
\\
	D\left(g_{2,q}\right)_{\X}(\V)&= \tr{\D_q \V}
\\
	D\left(g_3\right)_{\X}(\V)&= \tr{\X^{-1}\V}
\\
	D\left(g_{4,r}\right)_{\X}(\V)&=\tr{\A_r\V\H_r\X}+\tr{\A_r\X\H_r\V}= 2\tr{\A_r\X\H_r\V}
\\
	D\left(g_{5,s}\right)_{\X}(\V)
	&=-2\tr{\X^{-1}\F_s\X^{-1}\G_s\X^{-1}\V} 
\\
	D\left(g_{6,m}\right)_{\X}(\V)
	&= \tr{\Q_m\X^{-1}\P_m\V} - \tr{\X^{-1}\P_m\X\Q_m\X^{-1}\V}
\end{align}

Recall that, Riemannian hessian $H^Rf(\X)$ is related to Euclidean Hessian $ Hf(\X)$ through following relation
\begin{align}
	H^Rf(\X)[\V]&=\X Hf(\X)[\V]\X+\frac{1}{2}\big[\V\grd f(\X)\X+\X\grd f(\X)\V\big]
\end{align}
therefore,
\begin{align}
	\left\langle H^Rf(\X)[\V],\V\right\rangle_{\X}&= \tr{Hf(\X)[\V]\V}+\tr{\V\grd \;f(\X)\V\X^{-1}}
\end{align}
The optimal $\lambda_{\V}$ which minimizes the second order approximation \eqref{second_ordr_approx} for geodesically convex function  requires the quantities $\tr{\grd \; f(\X)\V}$, $\tr{Hf(\X)[\V]\V}$ and $\tr{\V\grd \;f(\X)\V\X^{-1}}$.

Let's define the following variables to simplify the equations:
\begin{align}
	h_{1,p}(\X)&=\frac{d g}{d g_{1,p}}(\X);\quad 1\leq p\leq P \label{h1p}\\
	h_{2,q}(\X)&=\frac{d g}{d g_{2,q}}(\X); \quad 1\leq q \leq Q\\
	h_3(\X)&=\frac{d g}{dg_3}(\X)\\
	h_{4,r}(\X)&=\frac{d g}{d g_{4,r}}(\X); \quad 1\leq r\leq R\\
	h_{5,s}(\X)&=\frac{d g}{d g_{5,s}}(\X); \quad 1\leq s\leq S\\
	h_{6,m}(\X)&=\frac{d g}{d g_{6,m}}(\X); \quad 1\leq m \leq M\label{h6m}
\end{align}
\begin{align}
	\F_{1,p}(\X)&=h_{1,p}(\X)\X^{-1}\C_p\X^{-1} \label{f1,p}\\
	\F_{2,q}(\X)&=h_{2,q}(\X)\D_q \\
	\F_3(\X)& =	h_3(\X)\X^{-1} \\
	\F_{4,r}(\X)& =h_{4,r}(\X)\left[\A_r\X\H_r+\H_r\X\A_r\right] \\
	\F_{5,s}(\X)& =	h_{5,s}(\X)\Big[\X^{-1}\F_s\X^{-1}\G_s\X^{-1}+\X^{-1}\G_s\X^{-1}\F_s\X^{-1}\Big] \\
	\F_{6,m}(\X)&=h_{6,m}(\X)\left[\Q_m\X^{-1}\P_m + \P_m\X^{-1}\Q_m-\X^{-1}\P_m\X\Q_m\X^{-1} -\X^{-1}\Q_m\X\P_m\X^{-1} \right] \label{f6,m}
\end{align}

The functions $h_{1,p}(\X)$, $h_{2,q}(\X)$, $h_3(\X) $, $h_{4,r}(\X)$, $h_{5,s}(\X)$, $h_{6,m}(\X)$ are the functions of  $g_{1,p}(\X)$, $g_{2,q}(\X)$, $g_3(\X)$, $g_{4,r}(\X)$, $g_{5,s}(\X)$, $g_{6,m}(\X)$. By the definition of total derivation, directional derivatives of variables defined in \eqref{h1p}-\eqref{h6m} are 
\begin{align}
	D\left(h_{1,p}\right)_{\X}(\V)&=\sum_{p^{'}=1}^{P}\left[\frac{d h_{1,p}}{d g_{1,p^{'}}}(\X)\right]\left[D\left(g_{1,p^{'}}\right)_{\X}(\V)\right]+ \sum_{q=1}^{Q} \left[\frac{d h_{1,p}}{d g_{2,q}}(\X)\right] \left[D\left(g_{2,q}\right)_{\X}(\V)\right] \nonumber\\
	&  + \left[\frac{d h_{1,p}}{d g_{3}}(\X)\right]\left[D\left(g_{3}\right)_{\X}(\V)\right]+ \sum_{r=1}^{R}\left[\frac{d h_{1,p}}{d g_{4,r}}(\X)\right]\left[D\left(g_{4,r}\right)_{\X}(\V)\right] \nonumber\\
	& + \sum_{s=1}^{S}\left[\frac{d h_{1,p}}{d g_{5,s}}(\X)\right]\left[D\left(g_{5,s}\right)_{\X}(\V)\right] + \sum_{m=1}^{M}\left[\frac{d h_{1,p}}{d g_{6,m}}(\X)\right]\left[D\left(g_{6,m}\right)_{\X}(\V)\right]
	\\
	D\left(h_{2,q}\right)_{\X}(\V)&=\sum_{p=1}^{P}\left[\frac{d h_{2,q}}{d g_{1,p}}(\X)\right]\left[D\left(g_{1,p}\right)_{\X}(\V)\right]+ \sum_{q^{'}=1}^{Q} \left[\frac{d h_{2,q}}{d g_{2,q^{'}}}(\X)\right] \left[D\left(g_{2,q^{'}}\right)_{\X}(\V)\right] \nonumber\\
	&  + \left[\frac{d h_{2,q}}{d g_{3}}(\X)\right]\left[D\left(g_{3}\right)_{\X}(\V)\right]+ \sum_{r=1}^{R}\left[\frac{d h_{2,q}}{d g_{4,r}}(\X)\right]\left[D\left(g_{4,r}\right)_{\X}(\V)\right] \nonumber\\
	& + \sum_{s=1}^{S}\left[\frac{d h_{2,q}}{d g_{5,s}}(\X)\right]\left[D\left(g_{5,s}\right)_{\X}(\V)\right] + \sum_{m=1}^{M}\left[\frac{d h_{2,q}}{d g_{6,m}}(\X)\right]\left[D\left(g_{6,m}\right)_{\X}(\V)\right]
	\\
	D\left(h_3\right)_{\X}(\V)&=\sum_{p=1}^{P}\left[\frac{d h_3}{d g_{1,p}}(\X)\right]\left[D\left(g_{1,p}\right)_{\X}(\V)\right]+ \sum_{q=1}^{Q} \left[\frac{d h_3}{d g_{2,q}}(\X)\right] \left[D\left(g_{2,q}\right)_{\X}(\V)\right] \nonumber\\
	&  + \left[\frac{d h_3}{d g_{3}}(\X)\right]\left[D\left(g_{3}\right)_{\X}(\V)\right]+ \sum_{r=1}^{R}\left[\frac{d h_3}{d g_{4,r}}(\X)\right]\left[D\left(g_{4,r}\right)_{\X}(\V)\right] \nonumber\\
	& + \sum_{s=1}^{S}\left[\frac{d h_3}{d g_{5,s}}(\X)\right]\left[D\left(g_{5,s}\right)_{\X}(\V)\right] + \sum_{m=1}^{M}\left[\frac{d h_3}{d g_{6,m}}(\X)\right]\left[D\left(g_{6,m}\right)_{\X}(\V)\right]
	\\
	D\left(h_{4,r}\right)_{\X}(\V)&=\sum_{p=1}^{P}\left[\frac{d h_{4,r}}{d g_{1,p}}(\X)\right]\left[D\left(g_{1,p}\right)_{\X}(\V)\right]+ \sum_{q=1}^{Q} \left[\frac{d h_{4,r}}{d g_{2,q}}(\X)\right] \left[D\left(g_{2,q}\right)_{\X}(\V)\right] \nonumber\\
	&  + \left[\frac{d h_{4,r}}{d g_{3}}(\X)\right]\left[D\left(g_{3}\right)_{\X}(\V)\right]+ \sum_{r^{'}=1}^{R}\left[\frac{d h_{4,r}}{d g_{4,r^{'}}}(\X)\right]\left[D\left(g_{4,r^{'}}\right)_{\X}(\V)\right] \nonumber\\
	& + \sum_{s=1}^{S}\left[\frac{d h_{4,r}}{d g_{5,s}}(\X)\right]\left[D\left(g_{5,s}\right)_{\X}(\V)\right] + \sum_{m=1}^{M}\left[\frac{d h_{4,r}}{d g_{6,m}}(\X)\right]\left[D\left(g_{6,m}\right)_{\X}(\V)\right]
	\\
	D\left(h_{5,s}\right)_{\X}(\V)&=\sum_{p=1}^{P}\left[\frac{d h_{5,s}}{d g_{1,p}}(\X)\right]\left[D\left(g_{1,p}\right)_{\X}(\V)\right]+ \sum_{q=1}^{Q} \left[\frac{d h_{5,s}}{d g_{2,q}}(\X)\right] \left[D\left(g_{2,q}\right)_{\X}(\V)\right] \nonumber\\
	&  + \left[\frac{d h_{5,s}}{d g_{3}}(\X)\right]\left[D\left(g_{3}\right)_{\X}(\V)\right]+ \sum_{r=1}^{R}\left[\frac{d h_{5,s}}{d g_{4,r}}(\X)\right]\left[D\left(g_{4,r}\right)_{\X}(\V)\right] \nonumber\\
	& + \sum_{s^{'}=1}^{S}\left[\frac{d h_{5,s}}{d g_{5,s^{'}}}(\X)\right]\left[D\left(g_{5,s^{'}}\right)_{\X}(\V)\right] + \sum_{m=1}^{M}\left[\frac{d h_{5,s}}{d g_{6,m}}(\X)\right]\left[D\left(g_{6,m}\right)_{\X}(\V)\right]
	\\
	D\left(h_{6,m}\right)_{\X}(\V)&=\sum_{p=1}^{P}\left[\frac{d h_{6,m}}{d g_{1,p}}(\X)\right]\left[D\left(g_{1,p}\right)_{\X}(\V)\right]+ \sum_{q=1}^{Q} \left[\frac{d h_{6,m}}{d g_{2,q}}(\X)\right] \left[D\left(g_{2,q}\right)_{\X}(\V)\right] \nonumber\\
	&  + \left[\frac{d h_{6,m}}{d g_{3}}(\X)\right]\left[D\left(g_{3}\right)_{\X}(\V)\right]+ \sum_{r=1}^{R}\left[\frac{d h_{6,m}}{d g_{4,r}}(\X)\right]\left[D\left(g_{4,r}\right)_{\X}(\V)\right] \nonumber\\
	& + \sum_{s=1}^{S}\left[\frac{d h_{6,m}}{d g_{5,s}}(\X)\right]\left[D\left(g_{5,s}\right)_{\X}(\V)\right] + \sum_{m^{'}=1}^{M}\left[\frac{d h_{6,m}}{d g_{6,m^{'}}}(\X)\right]\left[D\left(g_{6,m^{'}}\right)_{\X}(\V)\right]
\end{align}
The gradient vector field $\F(\X)$ can be written in terms of $h_{1,p}(\X)$, $h_{2,q}(\X)$, $h_3(\X)$, $h_{4,r}(\X)$, $h_{5,s}(\X)$ and  $h_{6,m}(\X)$ as follows:
\begin{align}
	\F(\X)&=-\sum_{p=1}^{P}h_{1,p}(\X)\X^{-1}\C_p\X^{-1}+ \sum_{q=1}^{Q}h_{2,q}(\X)\D_q+\sum_{r=1}^{R}h_{4,r}(\X)\left(\A_r\X\H_r+\H_r\X\A_r\right)\nonumber\\
	&-\sum_{s=1}^{S}h_{5,s}(\X)\left(\X^{-1}\F_s\X^{-1}\G_s\X^{-1}+\X^{-1}\G_s\X^{-1}\F_s\X^{-1}\right) +h_3(\X)\X^{-1}\nonumber\\
	&+\frac{1}{2}\sum_{m=1}^{M}h_{6,m}(\X)\left(\Q_m\X^{-1}\P_m + \P_m\X^{-1}\Q_m-\X^{-1}\P_m\X\Q_m\X^{-1} -\X^{-1}\Q_m\X\P_m\X^{-1} \right)
\end{align} 
using relations \eqref{h1p}-\eqref{f6,m} we have that
\begin{align}
	\F(\X)&= -\sum_{p=1}^{P}\F_{1,p}(\X)+ \sum_{q=1}^{Q}\F_{2,q}(\X)+\F_3(\X)+\sum_{r=1}^{R}\F_{4,r}(\X)-\sum_{s=1}^{S}\F_{5,s}(\X)+\frac{1}{2}\sum_{m=1}^{M}\F_{6,m}(\X)
\end{align}

Let us denote covariant  derivative of vector field $\F(\X)$ with respect to vector field $\V$ in Euclidean geometry as $\nabla_{\V}^E\F(\X)$.  We will use Leibniz rule for differentiation to calculate rate of change along a particular direction $\V$. Therefore, the Euclidean Hessian can be written as
\begin{align}
	Hf(\X)[\V]&=-\sum_{p=1}^{P}	\nabla_{\V}^E\F_{1,p}(\X)+ \sum_{q=1}^{Q}\nabla_{\V}^E\F_{2,q}(\X)+\nabla_{\V}^E\F_3(\X)+\sum_{r=1}^{R}\nabla_{\V}^E\F_{4,r}(\X)\nonumber\\
	&-\sum_{s=1}^{S}\nabla_{\V}^E\F_{5,s}(\X)+ \frac{1}{2}\sum_{m=1}^{M}\nabla_{\V}^E\F_{6,m}(\X)\label{hess_rel}
\end{align}
where,

\begin{align}
	\nabla_{\V}^E\F_{1,p}(\X)&=\Big[D\left(h_{1,p}\right)_{\X}(\V)\Big] \X^{-1}\C_p\X^{-1} - h_{1,p}(\X)\left[\begin{array}{ll}
		\X^{-1}\V\X^{-1}\C_p\X^{-1} \\+\X^{-1}\C_p\X^{-1}\V\X^{-1}
	\end{array}\right]
	\\
	\nabla_{\V}^E\F_{2,q}(\X)&= \Big[D\left(h_{2,q}\right)_{\X}(\V)\Big] \D_q
	\\
	\nabla_{\V}^E\F_3(\X)&= \Big[D\left(h_3\right)_{\X}(\V)\Big]\X^{-1}- h_3(\X)\Big(\X^{-1}\V\X^{-1}\Big)
	\\
	\nabla_{\V}^E\F_{4,r}(\X)&= \Big[D\left(h_{4,r}\right)_{\X}(\V)\Big]\left(\A_r\X\H_r+\H_r\X\A_r\right) + h_{4,r}(\X)\Big(\A_r\V\H_r+\H_r\V\A_r\Big)
	\\
	\nabla_{\V}^E\F_{5,s}(\X)&= \Big[D\left(h_{5,s}\right)_{\X}(\V)\Big]\Big[\X^{-1}\F_s\X^{-1}\G_s\X^{-1}+\X^{-1}\G_s\X^{-1}\F_s\X^{-1}\Big]\nonumber\\
	&- h_{5,s}(\X) \left[\begin{array}{ll}
		\X^{-1}\V\X^{-1}\F_s\X^{-1}\G_s\X^{-1}+\X^{-1}\F_s\X^{-1}\V\X^{-1}\G_s\X^{-1}\\
		+\X^{-1}\F_s\X^{-1}\G_s\X^{-1}\V\X^{-1}+\X^{-1}\V\X^{-1}\G_s\X^{-1}\F_s\X^{-1}\\
		+\X^{-1}\G_s\X^{-1}\V\X^{-1}\F_s\X^{-1} +\X^{-1}\G_s\X^{-1}\F_s\X^{-1}\V\X^{-1}
	\end{array}\right]
	\\
	\nabla_{\V}^E\F_{6,m}(\X) &=\Big[D\left(h_{6,m}\right)_{\X}(\V)\Big]\left[\begin{array}{ll}\Q_m\X^{-1}\P_m + \P_m\X^{-1}\Q_m-\X^{-1}\P_m\X\Q_m\X^{-1} \\ -\X^{-1}\Q_m\X\P_m\X^{-1} 
\end{array}\right]\nonumber\\
	&+h_{6,m}(\X)\left[\begin{array}{ll}
		-\Q_m\X^{-1}\V\X^{-1}\P_m- \P_m\X^{-1}\V\X^{-1}\Q_m\\
		+\X^{-1}\V\X^{-1}\P_m\X\Q_m\X^{-1}-\X^{-1}\P_m \V\Q_m\X^{-1}\\+ \X^{-1}\P_m\X\Q_m\X^{-1}\V\X^{-1}+\X^{-1}\V\X^{-1}\Q_m\X\P_m\X^{-1}\\
		-\X^{-1}\Q_m \V\P_m\X^{-1}+\X^{-1}\Q_m\X\P_m\X^{-1}\V\X^{-1}
	\end{array}\right]
\end{align}

Let us first define few intermediate variables
\begin{align}
	T_{11p}&=\tr{\X^{-1}\C_p\X^{-1}\V};&
	T_{12p}&=\tr{\X^{-1}\C_p\X^{-1}\V\X^{-1}\V};
	\\
	T_{21q}&=\tr{\D_q\V};&
	T_{22q}&=\tr{\D_q\V\X^{-1}\V };
	\\
	T_{31}&=\tr{\X^{-1}\V};&
	T_{32}&=\tr{\X^{-1}\V\X^{-1}\V};
	\\
	T_{41r}&=\tr{\A_r\X\H_r\V};&
	T_{42r}&=\tr{\A_r\V\H_r\V};\\
	T_{43r}&=\tr{\A_r\X\H_r\V\X^{-1}\V};
	\\
	T_{51s}&=\tr{\X^{-1}\F_s\X^{-1}\G_s\X^{-1}\V};&
	T_{52s}&=\tr{\X^{-1}\F_s\X^{-1}\G_s\X^{-1}\V\X^{-1}\V};\\
	T_{53s}&=\tr{\X^{-1}\F_s\X^{-1}\V\X^{-1}\G_s\X^{-1}\V};
	\\
	T_{61m}&=\tr{\Q_m\X^{-1}\P_m\V};&
	T_{62m}&=\tr{\P_m\X\Q_m\X^{-1}\V\X^{-1}};\\
	T_{63m}&=\tr{\P_m\V\Q_m\X^{-1}\V\X^{-1}};&
	T_{64m}&=\tr{\X^{-1}\P_m\X\Q_m\X^{-1}\V\X^{-1}\V};\\
	T_{65m}&=\tr{\Q_m\X^{-1}\P_m\V\X^{-1}\V};
\end{align}
which implies
\begin{align}
	D\left(g_{1,p}\right)_{\X}(\V)&= -\tr{\X^{-1}\C_p\X^{-1}\V} = -T_{11p}\\
	D\left(g_{2,q}\right)_{\X}(\V)&= \tr{\D_q \V} = T_{21q}\\
	D\left(g_3\right)_{\X}(\V)&= \tr{\X^{-1}\V} = T_{31} \\
	D\left(g_{4,r}\right)_{\X}(\V)&= 2\tr{\A_r\X\H_r\V} = 2T_{41r}\\
	D\left(g_{5,s}\right)_{\X}(\V)&=-2\tr{\X^{-1}\F_s\X^{-1}\G_s\X^{-1}\V}=-2T_{51s}\\
	D\left(g_{6,m}\right)_{\X}(\V)&= \tr{\Q_m\X^{-1}\P_m\V} - \tr{\X^{-1}\P_m\X\Q_m\X^{-1}\V}=T_{61m}-T_{62m}
\end{align}

From \eqref{hess_rel} we have that
\begin{align}
	\Rar	\tr{Hf(\X)[\V]\V}&=-\sum_{p=1}^{P}\tr{\nabla_{\V}^E\F_{1,p}(\X)\V}+ \sum_{q=1}^{Q}\tr{\nabla_{\V}^E\F_{2,q}(\X)\V}+\tr{\nabla_{\V}^E\F_3(\X)\V}\nonumber\\
	&+\sum_{r=1}^{R}\tr{\nabla_{\V}^E\F_{4,r}(\X)\V}-\sum_{s=1}^{S}\tr{\nabla_{\V}^E\F_{5,s}(\X)\V} \nonumber\\
	&+ \frac{1}{2}\sum_{m=1}^{M}\tr{\nabla_{\V}^E\F_{6,m}(\X)\V}\label{hess_relation}
\end{align}
Each component of equations \eqref{hess_relation} can be written as
\begin{align}
	\tr{\nabla_{\V}^E\F_{1,p}(\X)\V}
	&=\Big[D\left(h_{1,p}\right)_{\X}(\V)\Big]\tr{\X^{-1}\C_p\X^{-1}\V}-2 h_{1,p}(\X)\tr{\X^{-1}\C_p\X^{-1}\V\X^{-1}\V}\label{hess_1}
\\
	\tr{\nabla_{\V}^E\F_{2,q}(\X)\V}&=\Big[D\left(h_{2,q}\right)_{\X}(\V)\Big]\tr{\D_q\V}
\\
	\tr{\nabla_{\V}^E\F_3(\X)\V}&= \Big[D\left(h_3\right)_{\X}(\V)\Big]\tr{\X^{-1}\V}- h_3(\X)\tr{\X^{-1}\V\X^{-1}\V}
\\
	\tr{\nabla_{\V}^E\F_{4,r}(\X)\V}
	&= 2\left[D\left(h_{4,r}\right)_{\X}(\V)\right]\tr{\A_r\X\H_r\V}+ 2h_{4,r}(\X)\tr{\A_r\V\H_r\V}
\\
	\tr{\nabla_{\V}^E\F_{5,s}(\X)\V} 
	&= 2\Big[D\left(h_{5,s}\right)_{\X}(\V)\Big]\tr{\X^{-1}\F_s\X^{-1}\G_s\X^{-1}\V}\nonumber\\
	&- h_{5,s}(\X) \left[\begin{array}{ll}
	4\tr{\X^{-1}\F_s\X^{-1}\G_s\X^{-1}\V\X^{-1}\V}\\+2\tr{\X^{-1}\F_s\X^{-1}\V\X^{-1}\G_s\X^{-1}\V}
	\end{array}\right] 
\\
	\tr{\nabla_{\V}^E\F_{6,m}(\X)\V} 
	&=2\Big[D\left(h_{6,m}\right)_{\X}(\V)\Big]\left[
	\tr{\Q_m\X^{-1}\P_m\V} -\tr{\X^{-1}\P_m\X\Q_m\X^{-1}\V}\right]\nonumber\\
	&+4h_{6,m}(\X)\left[\begin{array}{ll}
	-\tr{\P_m\V\Q_m\X^{-1}\V\X^{-1}}\\
	+\tr{\X^{-1}\P_m\X\Q_m\X^{-1}\V\X^{-1}\V}
	\end{array}
	\right]\label{hess_6}
\end{align}
combining \eqref{hess_1}-\eqref{hess_6}, we have that
\begin{align}
	&\tr{Hf(\X)[\V]\V}	\nonumber\\
	&=		-\sum_{p=1}^{P}\tr{\nabla_{\V}^E\F_{1,p}(\X)\V}+ \sum_{q=1}^{Q}\tr{\nabla_{\V}^E\F_{2,q}(\X)\V}+\tr{\nabla_{\V}^E\F_3(\X)\V}\nonumber\\
	&+\sum_{r=1}^{R}\tr{\nabla_{\V}^E\F_{4,r}(\X)\V}-\sum_{s=1}^{S}\tr{\nabla_{\V}^E\F_{5,s}(\X)\V} + \frac{1}{2}\sum_{m=1}^{M}\tr{\nabla_{\V}^E\F_{6,m}(\X)\V}
	\nonumber\\
	\nonumber\\
	&=\left[\begin{array}{ll}
		-\sum_{p=1}^{P}\left[\Big[D\left(h_{1,p}\right)_{\X}(\V)\Big]\tr{\X^{-1}\C_p\X^{-1}\V}-2 h_{1,p}(\X)\tr{\X^{-1}\C_p\X^{-1}\V\X^{-1}\V}\right]\\
		+\sum_{q=1}^{Q}\left[\Big[D\left(h_{2,q}\right)_{\X}(\V)\Big]\tr{\D_q\V}\right]\\
		+\left[\Big[D\left(h_3\right)_{\X}(\V)\Big]\tr{\X^{-1}\V}- h_3(\X)\tr{\X^{-1}\V\X^{-1}\V}\right]\\
		+\sum_{r=1}^{R}\left[2\left[D\left(h_{4,r}\right)_{\X}(\V)\right]\tr{\A_r\X\H_r\V}+ 2h_{4,r}(\X)\tr{\A_r\V\H_r\V}\right]\\
		-\sum_{s=1}^{S}\left[\begin{array}{ll}
			2\Big[D\left(h_{5,s}\right)_{\X}(\V)\Big]\tr{\X^{-1}\F_s\X^{-1}\G_s\X^{-1}\V}\\
			- h_{5,s}(\X) \left[\begin{array}{ll}
				4\tr{\X^{-1}\F_s\X^{-1}\G_s\X^{-1}\V\X^{-1}\V}\\
				+2\tr{\X^{-1}\F_s\X^{-1}\V\X^{-1}\G_s\X^{-1}\V}
			\end{array}\right]
		\end{array}\right] \\
		+ \sum_{m=1}^{M}\left[\begin{array}{ll}
			\Big[D\left(h_{6,m}\right)_{\X}(\V)\Big]\left[
			\tr{\Q_m\X^{-1}\P_m\V} -\tr{\P_m\X\Q_m\X^{-1}\V\X^{-1}}\right]\\
			+2h_{6,m}(\X)\left[\begin{array}{ll} 
				-\tr{\P_m\V\Q_m\X^{-1}\V\X^{-1}}\\
				+\tr{\X^{-1}\P_m\X\Q_m\X^{-1}\V\X^{-1}\V}
			\end{array}
			\right]
		\end{array}\right]
	\end{array}\right]
	\nonumber\\
	\nonumber\\
	&=\left[\begin{array}{ll}
		-\sum_{p=1}^{P}\left[\Big[D\left(h_{1,p}\right)_{\X}(\V)\Big]T_{11p}-2 h_{1,p}(\X)T_{12p}\right]
		+\sum_{q=1}^{Q}\left[\Big[D\left(h_{2,q}\right)_{\X}(\V)\Big]T_{21q}\right]\\
		+\left[\Big[D\left(h_3\right)_{\X}(\V)\Big]T_{31}- h_3(\X)T_{32}\right]
		+\sum_{r=1}^{R}\left[2\left[D\left(h_{4,r}\right)_{\X}(\V)\right]T_{41r}+ 2h_{4,r}(\X)T_{42r}\right]\\
		-\sum_{s=1}^{S}\left[
		2\Big[D\left(h_{5,s}\right)_{\X}(\V)\Big]T_{51s}
		- h_{5,s}(\X) \left[4T_{52s}+2T_{53s}\right]
		\right] \\
		+ \sum_{m=1}^{M}\left[
		\Big[D\left(h_{6,m}\right)_{\X}(\V)\Big]\left[
		T_{61m} -T_{62m}\right] +2h_{6,m}(\X)\left[-T_{63m}+T_{64m}\right]
		\right]
	\end{array}\right]
\end{align}
Next we calculate 
\begin{align}
	\tr{\V\grd \;f(\X)\V\X^{-1}}&= -\sum_{p=1}^{P}\tr{\V\F_{1,p}(\X)\V\X^{-1}}+ \sum_{q=1}^{Q}\tr{\V\F_{2,q}(\X)\V\X^{-1}}\nonumber\\
	&+\tr{\V\F_3(\X)\V\X^{-1}}
	+\sum_{r=1}^{R}\tr{\V\F_{4,r}(\X)\V\X^{-1}}\nonumber\\
	&-\sum_{s=1}^{S}\tr{\V\F_{5,s}(\X)\V\X^{-1}}+\sum_{m=1}^{M}\tr{\V\F_{6,m}(\X)\V\X^{-1}} \label{grd_relation}
\end{align}
each of its components can be simplified as 
\begin{align}
	\tr{\V\F_{1,p}(\X)\V\X^{-1}}
	&=  h_{1,p}(\X)\tr{\X^{-1}\C_p\X^{-1}\V\X^{-1}\V} \label{vgrdvxin_1}\\
	\tr{\V\F_{2,q}(\X)\V\X^{-1}}
	&= h_{2,q}(\X)\tr{\D_q\V\X^{-1}\V }\\
	\tr{\V\F_3(\X)\V\X^{-1}}
	&= h_3(\X) \tr{\V \X^{-1}\V\X^{-1}}\\
	\tr{\V\F_{4,r}(\X)\V\X^{-1}}
	&=  2h_{4,r}(\X) \tr{\A_r\X\H_r\V\X^{-1}\V}\\
	\tr{\V\F_{5,s}(\X)\V\X^{-1}}
	&= 2h_{5,s}(\X) \Big[\tr{\X^{-1}\F_s\X^{-1}\G_s\X^{-1}\V\X^{-1}\V}\Big]\\
	\tr{\V\F_{6,m}(\X)\V\X^{-1}}
	&=2h_{6,m}(\X)\left[\begin{array}{ll}
		\tr{\Q_m\X^{-1}\P_m\V\X^{-1}\V}\\-\tr{\X^{-1}\P_m\X\Q_m\X^{-1}\V\X^{-1}\V}
	\end{array} 
	\right]\label{vgrdvxin_6}
\end{align}
combining \eqref{vgrdvxin_1}-\eqref{vgrdvxin_6}, we have that
\begin{align}
	\tr{\V\grd \;f(\X)\V\X^{-1}}&= \left[\begin{array}{ll}
		-\sum_{p=1}^{P}\tr{\V\F_{1,p}(\X)\V\X^{-1}}+ \sum_{q=1}^{Q}\tr{\V\F_{2,q}(\X)\V\X^{-1}}\\+\tr{\V\F_3(\X)\V\X^{-1}}
		+\sum_{r=1}^{R}\tr{\V\F_{4,r}(\X)\V\X^{-1}}\\-\sum_{s=1}^{S}\tr{\V\F_{5,s}(\X)\V\X^{-1}}+\frac{1}{2}\sum_{m=1}^{M}\tr{\V\F_{6,m}(\X)\V\X^{-1}}
	\end{array} \right]
	\nonumber
	\\
	&= \left[\begin{array}{ll}
		-\sum_{p=1}^{P}\left[h_{1,p}(\X)\tr{\X^{-1}\C_p\X^{-1}\V\X^{-1}\V}\right]\\
		+ \sum_{q=1}^{Q}\left[h_{2,q}(\X)\tr{\D_q\V\X^{-1}\V }\right]
		+\left[h_3(\X) \tr{\V \X^{-1}\V\X^{-1}}\right]\\
		+\sum_{r=1}^{R}\left[2h_{4,r}(\X) \tr{\A_r\X\H_r\V\X^{-1}\V}\right]\\
		-\sum_{s=1}^{S}\left[2h_{5,s}(\X) \Big[\tr{\X^{-1}\F_s\X^{-1}\G_s\X^{-1}\V\X^{-1}\V}\Big]\right]\\
		+\sum_{m=1}^{M}h_{6,m}(\X)\left[\begin{array}{ll}
			\tr{\Q_m\X^{-1}\P_m\V\X^{-1}\V}\\
			-\tr{\X^{-1}\P_m\X\Q_m\X^{-1}\V\X^{-1}\V} 
		\end{array}
		\right]
	\end{array}\right] 
	\nonumber
	\\
	&= \left[\begin{array}{ll}
		-\sum_{p=1}^{P}\left[h_{1,p}(\X)T_{12p}\right]	+ \sum_{q=1}^{Q}\left[h_{2,q}(\X)T_{22q}\right] 
		+\left[h_3(\X) T_{32}\right]\\
		+\sum_{r=1}^{R}\left[2h_{4,r}(\X) T_{43r}\right]
		-\sum_{s=1}^{S}\left[2h_{5,s}(\X) T_{52s}\right]\\
		+\sum_{m=1}^{M}h_{6,m}(\X)\left[
		T_{65m}	-T_{64m}	\right]
	\end{array}\right] 
\end{align}
finally, we calculate
\begin{align}
	\tr{\F(\X)\V} &= \tr{\grd \; f(\X)\V}
	\nonumber
	\\
	&=\left[\begin{array}{ll}
		-\sum_{p=1}^{P}h_{1,p}(\X)\tr{\X^{-1}\C_p\X^{-1}\V}+ \sum_{q=1}^{Q}h_{2,q}(\X)\tr{\D_q\V}+h_3(\X)\tr{\X^{-1}\V}\\
		+\sum_{r=1}^{R}2h_{4,r}(\X)\tr{\A_r\X\H_r\V}
		-\sum_{s=1}^{S}2h_{5,s}(\X)\tr{\X^{-1}\F_s\X^{-1}\G_s\X^{-1}\V}\\
		+\sum_{m=1}^{M}h_{6,m}(\X)\left(\begin{array}{ll}
			\tr{\Q_m\X^{-1}\P_m\V} -\tr{\P_m\X\Q_m\X^{-1}\V\X^{-1}} 
		\end{array} \right)
	\end{array}\right]
	\nonumber
	\\
	&=\left[\begin{array}{ll}
		-\sum_{p=1}^{P}h_{1,p}(\X)T_{11p}+ \sum_{q=1}^{Q}h_{2,q}(\X)T_{21q}+h_3(\X)T_{31}
		+\sum_{r=1}^{R}2h_{4,r}(\X)T_{41r}\\
		-\sum_{s=1}^{S}2h_{5,s}(\X)T_{51s}
		+\sum_{m=1}^{M}h_{6,m}(\X)\left[
		T_{61m} -T_{62m} \right]
	\end{array}\right]
\end{align}

All the above calculations can be further simplified when $\V$ is a basis vector as shown below. Let us divide it into two cases:

\textbf{Case $i \neq j$. }
\begin{align}
	\V&=\frac{1}{\sqrt{2}}\B\left(\e_i\e_j^{\T}+\e_j\e_i^{\T}\right)\B^{\T}
\end{align}

\begin{align}
	T_{11p}&=	\tr{\X^{-1}\C_p\X^{-1}\V}
	=\sqrt{2}\Big[\M_{1,p}(\X)\Big]_{ij}
	\\
	T_{12p}&=\tr{\X^{-1}\C_p\X^{-1}\V\X^{-1}\V}
	=\frac{1}{2}\Big(\left[\M_{1,p}(\X)\right]_{ii}+\left[\M_{1,p}(\X)\right]_{jj}\Big)
	\\
	T_{21q}&=	\tr{\D_q\V}
	=\sqrt{2}\Big[\M_{2,q}(\X)\Big]_{ij}
	\\
	T_{22q}&=	\tr{\D_q\V\X^{-1}\V }
	=\frac{1}{2}\Big(\left[\M_{2,q}(\X)\right]_{ii}+\left[\M_{2,q}(\X)\right]_{jj}\Big)
	\\
	T_{31}&=	\tr{\X^{-1}\V}
	=0
	\\
	T_{32}&	 = \tr{\X^{-1}\V\X^{-1}\V}
	=1
	\\
	T_{41r}&=	\tr{\A_r\X\H_r\V}
	=\frac{1}{\sqrt{2}}\Big(\Big[\M_{4,1,r}(\X)\Big]_{j:}\Big[\M_{4,2,r}(\X)\Big]_{:i}+\Big[\M_{4,1,r}(\X)\Big]_{i:}\Big[\M_{4,2,r}(\X)\Big]_{:j}\Big)
	\\
	T_{42r}&=	\tr{\A_r\V\H_r\V} 	\nonumber
	\\
	&
	=\frac{1}{2}\left(\begin{array}{ll}2\Big[\M_{4,1,r}(\X)\Big]_{ij}\Big[\M_{4,2,r}(\X)\Big]_{ij}+ \Big[\M_{4,1,r}(\X)\Big]_{ii}\Big[\M_{4,2,r}(\X)\Big]_{jj}\\
		+\Big[\M_{4,1,r}(\X)\Big]_{jj}\Big[\M_{4,2,r}(\X)\Big]_{ii}
	\end{array}\right)
	\\
	T_{43r}&=	\tr{\A_r\X\H_r\V\X^{-1}\V}
	= \frac{1}{2}\Big(\Big[\M_{4,1,r}(\X)\Big]_{i:}\Big[\M_{4,2,r}(\X)\Big]_{:i} + \Big[\M_{4,1,r}(\X)\Big]_{j:}\Big[\M_{4,2,r}(\X)\Big]_{:j}\Big)
	\\
	T_{51s}&=	\tr{\X^{-1}\F_s\X^{-1}\G_s\X^{-1}\V}	\nonumber
	\\
	&
	=\frac{1}{\sqrt{2}}\Big(\Big[\M_{5,1,s}(\X)\Big]_{i:}\Big[\M_{5,2,s}(\X)\Big]_{:j}+\Big[\M_{5,1,s}(\X)\Big]_{j:}\Big[\M_{5,2,s}(\X)\Big]_{:i}\Big)
	\\
	T_{52s}&=	\tr{\X^{-1}\F_s\X^{-1}\G_s\X^{-1}\V\X^{-1}\V}	\nonumber
	\\
	&
	=\frac{1}{2}\Big(\Big[\M_{5,1,s}(\X)\Big]_{i:}\Big[\M_{5,2,s}(\X)\Big]_{:i}+\Big[\M_{5,1,s}(\X)\Big]_{j:}\Big[\M_{5,2,s}(\X)\Big]_{:j}\Big)
	\\
	T_{53s}&=	\tr{\X^{-1}\F_s\X^{-1}\V\X^{-1}\G_s\X^{-1}\V}	\nonumber
	\\
	&
	=\frac{1}{2}\left(\begin{array}{ll}2\Big[\M_{5,1,s}(\X)\Big]_{ij}\Big[\M_{5,2,s}(\X)\Big]_{ij}+ \Big[\M_{5,1,s}(\X)\Big]_{ii}\Big[\M_{5,2,s}(\X)\Big]_{jj}\\
		+\Big[\M_{5,1,s}(\X)\Big]_{jj}\Big[\M_{5,2,s}(\X)\Big]_{ii}
	\end{array}\right)
	\\
	T_{61m}&=	\tr{\Q_m\X^{-1}\P_m\V}	\nonumber
	\\
	& =\frac{1}{\sqrt{2}}\left(\left[\left[\M_{6,2,m} (\X_t)\right]_{:j}\right]^{\T}\left[\M_{6,1,m}(\X_t)\right]_{:i}+ \left[\left[\M_{6,2,m} (\X_t)\right]_{:i}\right]^{\T}\left[\M_{6,1,m}(\X_t)\right]_{:j}\right)
	\\
	T_{62m}&=	\tr{\P_m\X\Q_m\X^{-1}\V\X^{-1}}	\nonumber
	\\
	&= \frac{1}{\sqrt{2}} \left(\left[\M_{6,1,m}(\X_t)\right]_{j:}\left[\left[\M_{6,2,m} (\X_t)\right]_{i:}\right]^{\T}+\left[\M_{6,1,m}(\X_t)\right]_{i:}\left[\left[\M_{6,2,m} (\X_t)\right]_{j:}\right]^{\T}\right)
	\\
	T_{63m}&=	\tr{\P_m\V\Q_m\X^{-1}\V\X^{-1}}	\nonumber
	\\
	&= \frac{1}{2}\tr{\begin{array}{ll}
				\left[\M_{6,1,m} (\X)\right]_{ji}  \left[\M_{6,2,m} (\X)\right]_{ij}+ 
				\left[\M_{6,1,m} (\X)\right]_{ii}  \left[\M_{6,2,m} (\X)\right]_{jj}\\
				+
				\left[\M_{6,1,m} (\X)\right]_{jj}  \left[\M_{6,2,m} (\X)\right]_{ii}+
				\left[\M_{6,1,m} (\X)\right]_{ij}  \left[\M_{6,2,m} (\X)\right]_{ji}
	\end{array}}  \\
	T_{64m}&=	\tr{\X^{-1}\P_m\X\Q_m\X^{-1}\V\X^{-1}\V}	\nonumber
	\\
	&=\frac{1}{2}\left(\left[\M_{6,1,m}(\X_t)\right]_{i:}\left[\left[\M_{6,2,m} (\X_t)\right]_{i:}\right]^{\T}+\left[\M_{6,1,m}(\X_t)\right]_{j:}\left[\left[\M_{6,2,m} (\X_t)\right]_{j:}\right]^{\T}\right)\\
	T_{65m}&=	\tr{\Q_m\X^{-1}\P_m\V\X^{-1}\V}	\nonumber
	\\
	&=\frac{1}{2}\left(\left[\left[\M_{6,2,m}(\X_t)\right]_{:i}\right]^{\T}\left[\left[\M_{6,1,m} (\X_t)\right]_{:i}\right]+\left[\left[\M_{6,2,m}(\X_t)\right]_{:j}\right]^{\T}\left[\left[\M_{6,1,m} (\X_t)\right]_{:j}\right]\right)
\end{align}

\textbf{Case $i = j$. }
\begin{align}
	\V&=\B\e_i\e_i^{\T}\B^{\T}
\end{align}

\begin{align}
	T_{11p}&=	\tr{\X^{-1}\C_p\X^{-1}\V} 
	=\Big[\M_{1,p}(\X)\Big]_{ii}
	\\
	T_{12p}&=	\tr{\X^{-1}\C_p\X^{-1}\V\X^{-1}\V}
	=\Big[\M_{1,p}(\X)\Big]_{ii}
	\\
	T_{21q}&=	\tr{\D_q\V}
	=\Big[\M_{2,q}(\X)\Big]_{ii}
	\\
	T_{22q}&=	\tr{\D_q\V\X^{-1}\V}
	=\Big[\M_{2,q}(\X)\Big]_{ii}
	\\
	T_{31}&=	\tr{\X^{-1}\V}
	=1
	\\
	T_{32}&=	\tr{\X^{-1}\V\X^{-1}\V}
	=1
	\\
	T_{41r}&=	\tr{\A_r\X\H_r\V}
	=\Big[\M_{4,1,r}(\X)\Big]_{i:}\Big[\M_{4,2,r}(\X)\Big]_{:i}
	\\
	T_{42r}&=	\tr{\A_r\V\H_r\V}
	=\Big[\M_{4,1,r}(\X)\Big]_{ii}\Big[\M_{4,2,r}(\X)\Big]_{ii}
	\\
	T_{43r}&=	\tr{\A_r\X\H_r\V\X^{-1}\V}
	=\Big[\M_{4,1,r}(\X)\Big]_{i:}\Big[\M_{4,2,r}(\X)\Big]_{:i}
	\\
	T_{51s}&=	\tr{\X^{-1}\F_s\X^{-1}\G_s\X^{-1}\V}
	=\Big[\M_{5,1,s}(\X)\Big]_{i:}\Big[\M_{5,2,s}(\X)\Big]_{:i}
	\\
	T_{52s}&=	\tr{\X^{-1}\F_s\X^{-1}\G_s\X^{-1}\V\X^{-1}\V}
	=\Big[\M_{5,1,s}(\X)\Big]_{i:}\Big[\M_{5,2,s}(\X)\Big]_{:i}
	\\
	T_{53s}&=	\tr{\X^{-1}\F_s\X^{-1}\V\X^{-1}\G_s\X^{-1}\V}
	=\Big[\M_{5,1,s}(\X)\Big]_{ii}\Big[\M_{5,2,s}(\X)\Big]_{ii}
	\\
	T_{61m}&=	\tr{\Q_m\X^{-1}\P_m\V}
	=\left[\left[\M_{6,2,m} (\X_t)\right]_{:i}\right]^{\T}\left[\M_{6,1,m}(\X_t)\right]_{:i}
	\\
	T_{62m}&=	\tr{\P_m\X\Q_m\X^{-1}\V\X^{-1}}
	=  \left[\M_{6,1,m}(\X_t)\right]_{i:}\left[\left[\M_{6,2,m} (\X_t)\right]_{i:}\right]^{\T}
	\\
	T_{63m}&=	\tr{\P_m\V\Q_m\X^{-1}\V\X^{-1}}
	= \left[\M_{6,2,m} (\X)\right]_{ii}  \left[\M_{6,1,m} (\X)\right]_{ii}
	\\
	T_{64m}&=	\tr{\X^{-1}\P_m\X\Q_m\X^{-1}\V\X^{-1}\V}
	=\Big[\M_{6,1,m}(\X)\Big]_{i:}\left[\left[\M_{6,2,m} (\X)\right]_{i:}\right]^{\T}
	\\
	T_{65m}&=	\tr{\Q_m\X^{-1}\P_m\V\X^{-1}\V}
	=\left[\Big[\M_{6,2,m}(\X)\Big]_{:i}\right]^{\T}\Big[\M_{6,1,m} (\X)\Big]_{:i}
\end{align}

\section{Matrix geometric mean of  two SPD matrices}\label{MGM_two}
The solution of matrix geometric mean problem $
	\min_{\X \in \Pn}\sum_{i=1}^{2}\left\|\log\left(\X^{-1/2}\mathbf{W}_i\X^{-1/2}\right)\right\|_F^2$ is 
\begin{align}
	\X_{g}^{\star}&=\W_1^{\frac{1}{2}}\left(\W_1^{-\frac{1}{2}}\W_2\W_1^{-\frac{1}{2}}\right)^{\frac{1}{2}} \W_1^{\frac{1}{2}}
\end{align}
also note that 
\begin{align}
	\W_2&=\W_1^{\frac{1}{2}}\left(\W_1^{-\frac{1}{2}}\W_2\W_1^{-\frac{1}{2}}\right) \W_1^{\frac{1}{2}}
\end{align}
The function $f(\X)=\sum_{i=1}^{2}\tr{\W_i\X^{-1}+\W_i^{-1}\X}$ is geodesically strongly convex for $\W_i\succ 0$ and its gradient is given as
	\begin{align}
		\grd f(\X)&=\sum_{i=1}^{2}\left(-\X^{-1}\W_i\X^{-1}+\W_i^{-1}\right)
	\end{align}
At the minimizer $\X^*$ of $f(\X)$, $\grd f(\X^*)=0$. Instead, let us evaluate $\grd f(\X)$ at $\X_{g}^{\star}$
	\begin{align}
		\grd f(\X_{g}^{\star})&=\sum_{i=1}^{2}\left(-(\X_{g}^{\star})^{-1}\W_i(\X_{g}^{\star})^{-1}+\W_i^{-1}\right)\\
		&=-(\X_{g}^{\star})^{-1}\W_1(\X_{g}^{\star})^{-1}+\W_1^{-1}-(\X_{g}^{\star})^{-1}\W_2(\X_{g}^{\star})^{-1}+\W_2^{-1}\\
		&=-\W_1^{-\frac{1}{2}}\left[\W_1^{-\frac{1}{2}}\W_2\W_1^{-\frac{1}{2}}\right]^{-\frac{1}{2}} \W_1^{-\frac{1}{2}}\W_1\W_1^{-\frac{1}{2}}\left[\W_1^{-\frac{1}{2}}\W_2\W_1^{-\frac{1}{2}}\right]^{-\frac{1}{2}} \W_1^{-\frac{1}{2}}+\W_1^{-1}\nonumber\\
		&-\left[\begin{array}{ll}
		\W_1^{-\frac{1}{2}}\left(\W_1^{-\frac{1}{2}}\W_2\W_1^{-\frac{1}{2}}\right)^{-\frac{1}{2}} \W_1^{-\frac{1}{2}}\W_1^{\frac{1}{2}}\left(\W_1^{-\frac{1}{2}}\W_2\W_1^{-\frac{1}{2}}\right) \W_1^{\frac{1}{2}}
		\\
	 \times
	 \W_1^{-\frac{1}{2}}\left(\W_1^{-\frac{1}{2}}\W_2\W_1^{-\frac{1}{2}}\right)^{-\frac{1}{2}} \W_1^{-\frac{1}{2}}
		\end{array}\right]+\W_2^{-1}\\
	&=-\W_2^{-1}+\W_1^{-1}-\W_1^{-1}+\W_2^{-1}=0
	\end{align}
Therefore, because of the geodesically strong convexity of $f(\X)$, the matrix geometric mean of two symmetric positive definite (SPD) matrices is the unique minimizer of the function $f(\X)$.

\vskip 0.2in
\bibliography{maniopt_ref_Arxiv}

\end{document}